\documentclass[algo2e,10pt,twocolumn]{IEEEtran}
\usepackage{soul,xcolor,lipsum,enumitem}
\usepackage[linesnumbered,ruled,vlined]{algorithm2e}

\SetCommentSty{mycommfont}
\usepackage{algorithmicx,comment,verbatim}
\usepackage{tablefootnote}
\DeclareMathAlphabet\mathbfcal{OMS}{cmsy}{b}{n}
\usepackage{algpseudocode}
\usepackage{amsfonts}
\usepackage{dirtytalk}
\usepackage{amsmath}
\usepackage{upgreek}
\usepackage{amssymb}
\usepackage{amsthm}
\usepackage[ansinew]{inputenc} 
\usepackage{dsfont}
\usepackage{mathtools}
\usepackage{graphicx}
\usepackage[skip=1pt]{subcaption}
\usepackage{import}
\usepackage{multirow}
\usepackage{cite}
\usepackage[export]{adjustbox}
\usepackage{breqn}
\usepackage{mathrsfs}
\usepackage{acronym}
\usepackage{setspace}
\usepackage{stackengine}
\usepackage{steinmetz}
\usepackage[skip=2pt,font=footnotesize]{caption}

\DeclareMathOperator*{\argmax}{arg\,max}
\newcommand{\inprod}[2]{\langle #1, #2\rangle}

\newtheorem{theorem}{Theorem}

\theoremstyle{definition}

\theoremstyle{remark}

\usepackage{cancel}

\graphicspath{{./figures/}}
\newcommand{\bfbeta}{\boldsymbol{\beta}}
\newcommand{\bfalpha}{\boldsymbol{\alpha}}
\newcommand{\secref}[1]{Sec.~\ref{#1}}
\newcommand{\figref}[1]{Fig.~\ref{#1}}
\newcommand\numberthis{\addtocounter{equation}{1}\tag{\theequation}}
\newcommand\nm[1]{}

\newcommand{\boldnu}{\boldsymbol \nu}
\newcommand{\boldpsi}{\boldsymbol \psi}

\title{
{Learning and Adaptation\\ for Millimeter-Wave
Beam Tracking and Training:\\ a Dual Timescale Variational Framework}
}

\author{Muddassar Hussain, \textit{Member, IEEE} and Nicol\`{o} Michelusi, \textit{Senior Member, IEEE}
\thanks{This work was supported in part by the National Science Foundation under grants CNS-1642982 and CNS-2129015.}	
	\thanks{M. Hussain is with the School of Electrical and Computer Engineering, Purdue University, West Lafayette, IN, USA; email: hussai13@purdue.edu. N. Michelusi is with the School of Electrical, Computer and Energy Engineering, Arizona State University, AZ, USA; email: nicolo.michelusi@asu.edu.}%
	\thanks{A preliminary version of this paper appeared at IEEE Globecom 2021 \cite{gcom}.}
	\vspace{-8mm}
}

\IEEEoverridecommandlockouts
\begin{document}
\setstcolor{red}
\setulcolor{red}
\setul{red}{2pt}

\maketitle

\begin{abstract}
Millimeter-wave vehicular networks incur enormous beam-training overhead to enable narrow-beam communications. This paper proposes a learning and adaptation framework in which the dynamics of the communication beams are learned and then exploited to design adaptive {beam-tracking and} training with low overhead: on a long-timescale, a deep recurrent variational autoencoder (DR-VAE) uses noisy beam-training feedback to learn a probabilistic model of beam dynamics {and enable predictive beam-tracking}; on a short-timescale, an adaptive beam-training procedure is formulated as a partially observable (PO-) Markov decision process (MDP) and optimized via \emph{point-based value iteration} (PBVI) by leveraging beam-training feedback and a probabilistic prediction of the strongest beam pair provided by the DR-VAE. In turn, beam-training feedback is used to refine the DR-VAE via stochastic gradient ascent in a continuous process of learning and adaptation. The proposed DR-VAE learning framework learns accurate beam dynamics: it reduces the Kullback-Leibler divergence between the ground truth and the learned model of beam dynamics by ${\sim}95\%$ over the Baum-Welch algorithm and a naive learning approach that neglects feedback errors. Numerical results on a line-of-sight scenario with multipath and 3D beamforming reveal that the proposed dual timescale approach yields near-optimal spectral efficiency, and improves it by $130\%$ over a policy that  scans exhaustively over the dominant beam pairs, and by $20\%$ over a state-of-the-art POMDP policy. Finally, a low-complexity policy is proposed by reducing the POMDP to an error-robust MDP, and is shown to perform well in regimes with infrequent feedback errors.\end{abstract}
\vspace{-5mm}
\section{Introduction} 
Millimeter-wave (mm-wave)  has emerged as the most promising technology to enable multi-gigabit rates in vehicular communications, thanks to large bandwidth availability~\cite{choi2016millimeter}. By operating at carrier frequencies ranging from $30$ up to hundreds of GHz, mm-wave communications suffer from high isotropic path loss, overcome
 by using large antenna arrays with beamforming \cite{rappaport_mmwave_book}. Yet,  highly directional communications are susceptible to beam misalignment due to the mobility of the user equipment (UE) or the surrounding propagation environment.
 Traditional beam-alignment schemes such as exhaustive search \cite{giordani2017millimeter} suffer from severe overhead, increased communication delay, and degraded spectral efficiency.

To achieve efficient design, adaptive beam-training schemes have been proposed \cite{va2016beam,scalabrin2018beam,TVT2020,javdi,TWC2019}, that leverage information on the UE's mobility and beam-training feedback.
In our recent work \cite{TVT2020}, we showed that statistical knowledge of the UE's mobility
is key to reduce the beam-training overhead and improve spectral efficiency, even in highly mobile V2X scenarios.  Yet, the design in \cite{TVT2020} relies on accurate statistical knowledge of the beam dynamics, which may need to be estimated from noisy measurements.
This observation begs the critical question: \emph{How can beam dynamics be estimated and leveraged to optimize beam-training and data communications?}

To address this challenge, in this paper we consider a mm-wave vehicular communication scenario, where a UE moves along a road
according to an \emph{unknown} mobility model and is served by a roadside base station (BS). Both UE and BS employ large antenna arrays with 3D beamforming to enable directional communication. The mobility of the UE and the surrounding environment induce dynamics in the \emph{strongest beam pair} that maximize the beamforming gain; these dynamics could be exploited to enable efficient beam-tracking and training. To learn and exploit these unknown dynamics,
 we propose a dual timescale learning and adaptation framework:
 in the long-timescale (of the order of several hundred frames), 
 the BS uses beam-training feedback to learn the \emph{strongest beam pair dynamics} and enable predictive beam-tracking,
 using a deep recurrent variational autoencoder (DR-VAE) \cite{rvae};
in the short-timescale (one frame duration),
 adaptive beam-training schemes leverage the probabilistic prediction of the strongest beam pair provided by the DR-VAE and
  beam-training feedback to minimize the overhead and maximize the frame spectral efficiency. In turn, beam-training measurements are used to refine the DR-VAE via stochastic gradient ascent, in a continuous process of learning and adaptation.

We formulate the decision-making process over the short-timescale as a partially observable (PO-) Markov decision process (MDP) and propose a \emph{point-based value iteration} (PBVI) method
that leverages provable structural properties
to design an approximately optimal policy, which provides the rule to select beam-training actions based on the belief (probability distribution over the optimal BS-UE beam pair, given the history of actions and beam-training measurements). 
To trade computational complexity with accuracy, we propose an MDP-based policy that operates under the assumption of error-free beam-training feedback, and an error-robust version that combines MDP-based actions with POMDP-based belief updates:
  it is shown that the error-robust MDP-based policy can be optimized
   at a fraction ($\sim$1/5) of the time while achieving spectral efficiency
close to the POMDP-based policy in regimes with infrequent feedback errors.

Numerical evaluations using 3D analog beamforming at both BS and UE reveal that the DR-VAE coupled with PBVI-based adaptation
reduce the 
average Kullback-Leibler (KL) divergence between a ground truth Markovian and the learned models
by $\sim$95\% over the Baum-Welch algorithms \cite{baum_welch}
and a naive approach that ignores errors in the beam-training measurements,
and improve spectral efficiency by $\sim$10\%.
 Finally, we simulate a setting with 2D UE mobility and 3D analog beamforming {for two scenarios, one with line-of-sight (LOS) channels, and the other with additional non-LOS (NLOS) multipath}:
the PBVI policy 
 improves the spectral efficiency over a state-of-the-art \emph{short-timescale single shot} (STSS) POMDP policy \cite{stss} by $16\%$ and $20\%$ in the two scenarios,
 and outperforms a scheme that scans exhaustively over the dominant beam pairs (EXOS) by $85\%$ and $130\%$, respectively. Similarly, the error-robust MDP policy improves the spectral efficiency by $8\%$ and $7\%$ over the STSS policy and  by $70\%$ and $100\%$ over EXOS, respectively.

{\bf Related Work}: 
Beam-alignment has been a topic of intense research in the last decade
and can be categorized as beam sweeping~\cite{michelusi2018optimal}, estimation of angles of arrival (AoA) and departure (AoD)~\cite{marzi}, and \mbox{contextual-information-aided} schemes~\cite{inverse_finger,ten_comp,jefflideep,va2018online,radar,sub6}. 
Despite their simplicity, these schemes do not incorporate mobility dynamics, leading to large beam-training overhead in high mobility scenarios~\cite{choi2016millimeter}. 
Moreover, beam-training may still be required to compensate for noise and inaccuracies in  contextual information,
or due to privacy concerns (e.g., sharing GPS coordinates as in \cite{inverse_finger,ten_comp,jefflideep,va2018online}).
 
 Recent works \cite{stss,alkhateeb2018deep,va2018online,second-best,javdi,TWC2019,nitindeep,jefflideep} proposed adaptive and machine learning-based solutions exploiting side-information and/or beam-training feedback.
  For instance,~\cite{alkhateeb2018deep} uses the received sounding signal from multiple surrounding BSs to predict the optimal beam via deep learning. In \cite{nitindeep}, a convolutional neural network is trained based on simulated channels and then used to predict beams via  compressive sensing.
  Reinforcement learning-based beam-alignment schemes have been proposed in \cite{stss,va2018online, second-best,javdi,TWC2019}.
  In our previous works \cite{TWC2019,second-best}, we used  the beam-training feedback to design adaptive beam-training policies under the assumption of error-free and erroneous beam-training feedback, respectively.
   In the aforementioned works, mobility is not leveraged in the design, incurring large overhead in high mobility scenarios~\cite{choi2016millimeter}.
 
 To overcome this limitation, in our recent work~\cite{TVT2020} we showed that, by exploiting the beam dynamics via POMDP, the spectral efficiency of V2X communication is greatly improved over conventional schemes, such as exhaustive search.
Similarly, the paper \cite{stss}
exploits beam dynamics to make beam predictions, and proposes a POMDP-based scheme (evaluated numerically in \secref{sec:numres}) to maximize the expected spectral efficiency.
Yet, these papers
assume a priori statistical knowledge of beam dynamics, which need to be learned
from noisy beam-training measurements in practice.
Without statistical knowledge of beam dynamics, \cite{TVT2020} is not amenable to
concurrent learning and optimization of the POMDP policy, due to large optimization cost. In contrast to~\cite{TVT2020}, 
herein, we decouple the beam-training design 
from the estimation of beam dynamics, by proposing a dual timescale approach in which a stochastic model of beam dynamics is learned on the long-timescale to enable predictive beam-tracking, interleaved with the execution of a beam-training policy in the short-timescale. The beam-training procedure is optimized in the short-timescale (single frame, agnostic to beam dynamics),
using beam predictions provided by the long-timescale learning module.
We note three advantages of this approach over~\cite{TVT2020,stss}:
1) since learning the model of beam dynamics is decoupled from the beam-training policy optimization,
learning and adaptation can be done concurrently (vs offline model learning required in \cite{TVT2020});
2) the maximization of the frame spectral efficiency  (vs average long-term in \cite{TVT2020}),
favors accurate detection of the optimal BS-UE's beam pair, 
which in turn improves the ability to predict optimal beam association for the next frames,
and indirectly maximizes spectral efficiency in the long-timescale;
3) by enabling adaptive beam-training  of \emph{variable} duration with \emph{multiple} beam-training rounds,
vs a \emph{single-shot} beam-training phase of \emph{fixed} duration of \cite{stss}, we achieve superior performance, as
shown numerically in \secref{sec:numres}.

\textbf{Contributions:} in a nutshell, we propose a dual timescale approach 
in which the dynamics of the strongest beam pair are learned over the long-timescale, to enable predictive beam-tracking
and efficient beam-training over the short-timescale:
\begin{enumerate}[leftmargin=*]
\item We propose a DR-VAE-based learning framework to learn the dynamics of the strongest beam pair, trained via stochastic gradient ascent from beam-training feedback;
\item We formulate a POMDP framework to design beam-training in the short-timescale,
with the goal of maximizing the expected frame spectral efficiency,
by leveraging predictions provided by the DR-VAE and beam-training feedback. We propose a linear time PBVI algorithm to find an approximately optimal policy, and leverage structural properties to reduce its complexity;
\item 
For the special case of error-free feedback, we prove that it is optimal to scan the most likely beam pairs, leading to a low-complexity value iteration algorithm. We design an error-robust version that combines MDP-based actions with POMDP-based belief updates,
and demonstrate near-optimal performance in regimes with infrequent  errors.
\end{enumerate}
\noindent
The rest of this paper is organized as follows:
\secref{sec:sys_model} presents the system model; 
\secref{sec:POMDP} describes the POMDP formulation;
 \secref{sec:vae} presents the DR-VAE learning framework;
 \secref{sec:numres} provides numerical results, followed by final remarks in \secref{sec:conc}.

 \textbf{Notation:} $\mathbb C$: set of complex numbers; $\mathcal{CN}$: complex Gaussian distribution;
 $\mathcal E(m)$: exponential distribution with mean $m$;
 $\mathcal G$: standard Gumbel distribution;
 $\mathbb E[\cdot]$: expectation operator;
 $\|
 \cdot\|_2$: L2 norm; $(\cdot)^{\rm H}$: complex conjugate transpose; $\odot$: Hadamard (element-wise) product of two vectors; $\mathbb I[\cdot]$: indicator function;
 $\inprod{\bfbeta }{\bfalpha}=\sum_{i=1}^n\bfbeta(i)\bfalpha(i)$: inner product of
 $\bfalpha,\bfbeta\in\mathbb R^n$;
$\mathbf e_i$: standard basis column vector with $\mathbf e_i(i){=}1$,
$\mathbf e_i(j){=}0,\forall j{\neq}i$;
$\mathcal B$: probability simplex (of suitable dimension);
{for a finite discrete set $\mathcal I$ and a function $g{:}\mathcal I{\mapsto}\mathbb R$,
 $\argmax_{i\in\mathcal I}g(i)$ returns an element 
 $i^*{\in}\mathcal I$ such that $g(i^*){=}\max_{i\in\mathcal I}g(i)$, with ties  broken arbitrarily.}
\vspace{-2mm}
\section{System Model}
\label{sec:sys_model}
We consider a mobile mm-wave vehicular system: 
 a UE is moving along the road, covered by multiple roadside base stations (BSs),
 and is served by one BS at a time;
  if it exits the coverage area of its serving BS, handover is performed to the next serving BS. In this paper, we restrict the design to the UE and its serving BS, as depicted in \figref{fig:sys}.
The BS and UE both use 3D beamforming with large antenna arrays (e.g, uniform planar arrays), with $M_{\rm tx}$ and $M_{\rm rx}$ antennas, respectively {(the main system parameters are summarized in Table
\ref{table1} of \secref{sec:numres})}.

We consider a frame-based system,
with frames of duration $T_{\rm fr}$ divided into $K$ slots, each of duration $T_{\rm s}{\triangleq}T_{\rm fr}/K$. 
To enable beamforming, each frame is split into a beam-training (BT) phase of variable number of slots, followed by data communication (DC) for the remainder of the frame, as shown in \figref{fig:two_timescale}.
The mobility of the UE and of the surrounding propagation environment induce dynamics in the \emph{optimal} beam pair that maximize the beamforming gain at the transmitter and receiver (see example of \figref{fig:sys}).
We assume that the frame duration does not exceed 
the \emph{beam coherence time} $T_{b}$ \cite{beam_coherence_time},
i.e., $T_{\rm fr}\leq T_{b}$, so that the optimal beam pair remain constant during the entire frame duration but may change across frames. {For example,
using the analysis of \cite{beam_coherence_time}
for a line of sight (LOS) scenario
with beams of $11.25^\circ$ beamwidth,
 operating at a carrier frequency of 30[GHz],
we find that $T_b {\simeq}400$[ms]
 for a  UE traveling at 108[km/h] at 50[m] from the BS, perpendicular to the beam direction, so that the choice of $T_{\rm fr}{=}20$[ms] (used also in our numerical evaluations) satisfies this assumption}.

\begin{figure}[t]
	\centering
	\includegraphics[trim = 0 0 0 0,clip,width=.85\columnwidth]{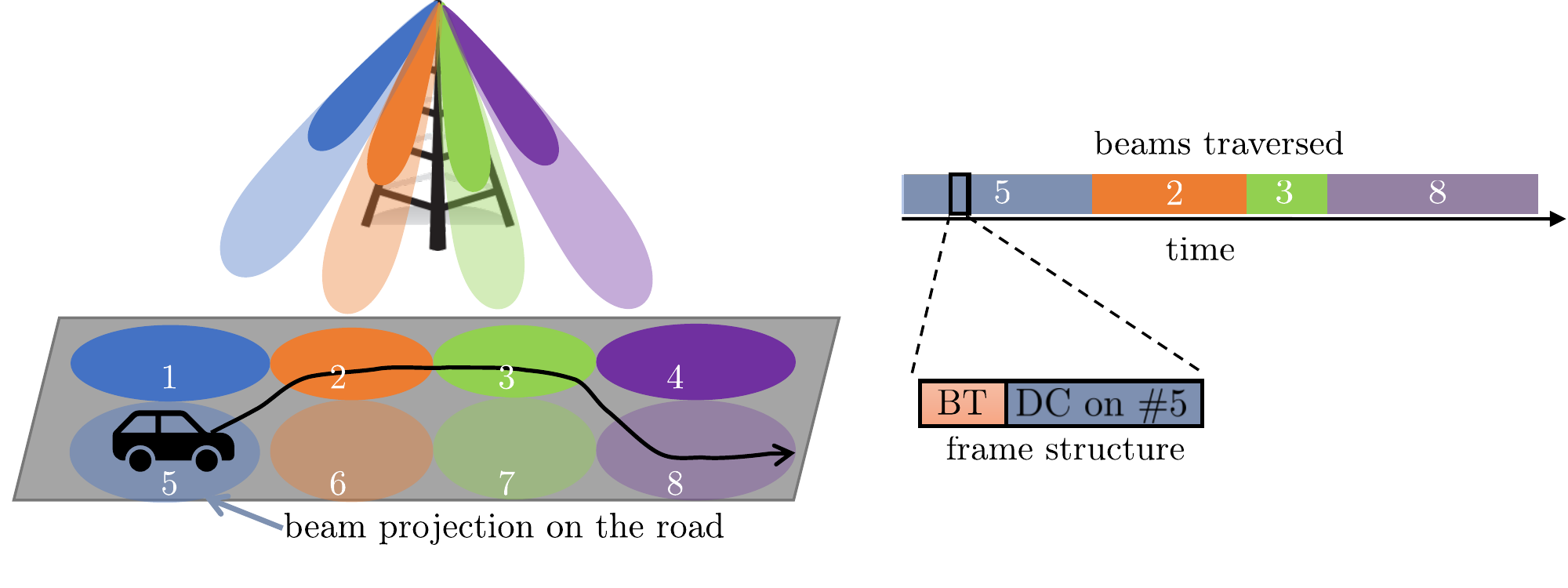}
	\vspace{-2mm}
	\caption{Example of a mm-wave {system with one BS and one omni-directional UE (multi-antenna case considered in the paper): as the UE moves across the coverage region, it traverses several communication beams (top right); to enable beamforming, each frame
	is split into a beam-training (BT) phase followed by data communication (DC) for the rest of the frame (bottom right).}}\label{fig:sys}
\end{figure}

The goal of this paper is to learn these beam-dynamics to enable predictive beam-tracking and design adaptive BT techniques with improved spectral efficiency.
To this end, we propose a \emph{dual timescale learning and adaptation framework}, depicted as a block diagram in \figref{fig:two_timescale}:
in the long-timescale (the time interval during which the UE stays within the BS' coverage area, of the order of several hundred frames)
a model of beam dynamics is learned, to enable predictive beam-tracking; in the short-timescale (one frame duration),
an adaptive BT policy is designed to maximize the expected frame spectral efficiency, by exploiting 
a probabilistic beam prediction (prior belief) provided by the beam dynamics
 learning module and BT feedback. Specifically:
\begin{enumerate}[leftmargin=*]
\item Modeling the dynamics of the strongest beam pair index (SBPI, the index of the beamforming vectors that maximize the {expected} beamforming gain, {averaged over small-scale conditions}) as a Markov process, the BS leverages a learned SBPI's transition model to
provide probabilistic predictions over SBPIs (prior belief)
 at the start of each frame;
\item The prior belief and BT feedback
are leveraged to optimize the
beam-training process on the short-timescale to maximize the spectral efficiency via POMDP (\secref{sec:POMDP});
\item The Markov SBPI model is  continuously learned via a DR-VAE-based framework (\secref{sec:vae}), based on BT feedback collected in the observation buffer.
\end{enumerate}
Next, we describe the 
model features and abstractions used.
\vspace{-4mm}
 \subsection{Signal and Channel Models}
 \vspace{-1mm}
 \label{sec:sig}
 Let $\mathbf x_{t,k}{\in}\mathbb C^{L_{\rm sy}\times 1}$ be the {sequence of} 
 $L_{\rm sy}$ symbols transmitted by the BS
 in slot $k{\in}\mathcal K{\triangleq} \{0,1,\cdots,K{-}1 \}$ of
 frame $t \in \mathbb N$,
   with $\mathbb E[\Vert \mathbf x_{t,k}\Vert_2^2]{=}L_{\rm sy}$. 
The corresponding signal received at the UE,
$\mathbf y_{t,k}\in\mathbb C^{L_{\rm sy}\times 1}$, is expressed as
\begin{align}
\label{eq:signal_model}
\mathbf y_{t,k} = \sqrt{P_{t,k}}\cdot(\mathbf f_{t,k}^{\rm H} \mathbf{H}_{t,k} \mathbf{c}_{t,k})\cdot{\mathbf x_{t,k}} +\mathbf w_{t,k},
\end{align}
where: $P_{t,k}$ is the average transmit power of the BS; 
{$\mathbf H_{t,k}{\in}\mathbb C^{M_{\rm rx}\times M_{\rm tx}}$ is the channel matrix;}
{$\mathbf w_{t,k}{\in}\mathbb C^{L_{\rm sy}\times 1}$} is the AWGN with variance $\sigma_w^2$, $\mathbf w_{t,k}{\sim} \mathcal{CN}(\boldsymbol 0, \sigma_w^2 \mathbf I)$;
 $\mathbf{c}_{t,k} {\in} \mathcal C$ and $\mathbf{f}_{t,k} {\in}\mathcal F$  are unit-norm beamforming vectors at the BS and UE,
  taking values from pre-designed analog beamforming codebooks $\mathcal C{\subset}  \mathbb{C}^{M_{\rm tx}\times 1}$ and $\mathcal F{\subset} \mathbb{C}^{M_{\rm rx}\times 1}$, respectively.
  We index 
  all possible beamforming vector pairs of the BS and UE with the \emph{beam pair index} (BPI)
$j{\in}\mathcal S{\triangleq}\{1,\cdots,|\mathcal C||\mathcal F|\}$, so that
$(\mathbf c^{(j)},\mathbf f^{(j)}){\in}\mathcal C {\times}\mathcal F$ is the $j$th beam pair.

\begin{figure}[t]
	\centering
	\includegraphics[trim = 0 10 0 0 ,clip,width=0.7\columnwidth]{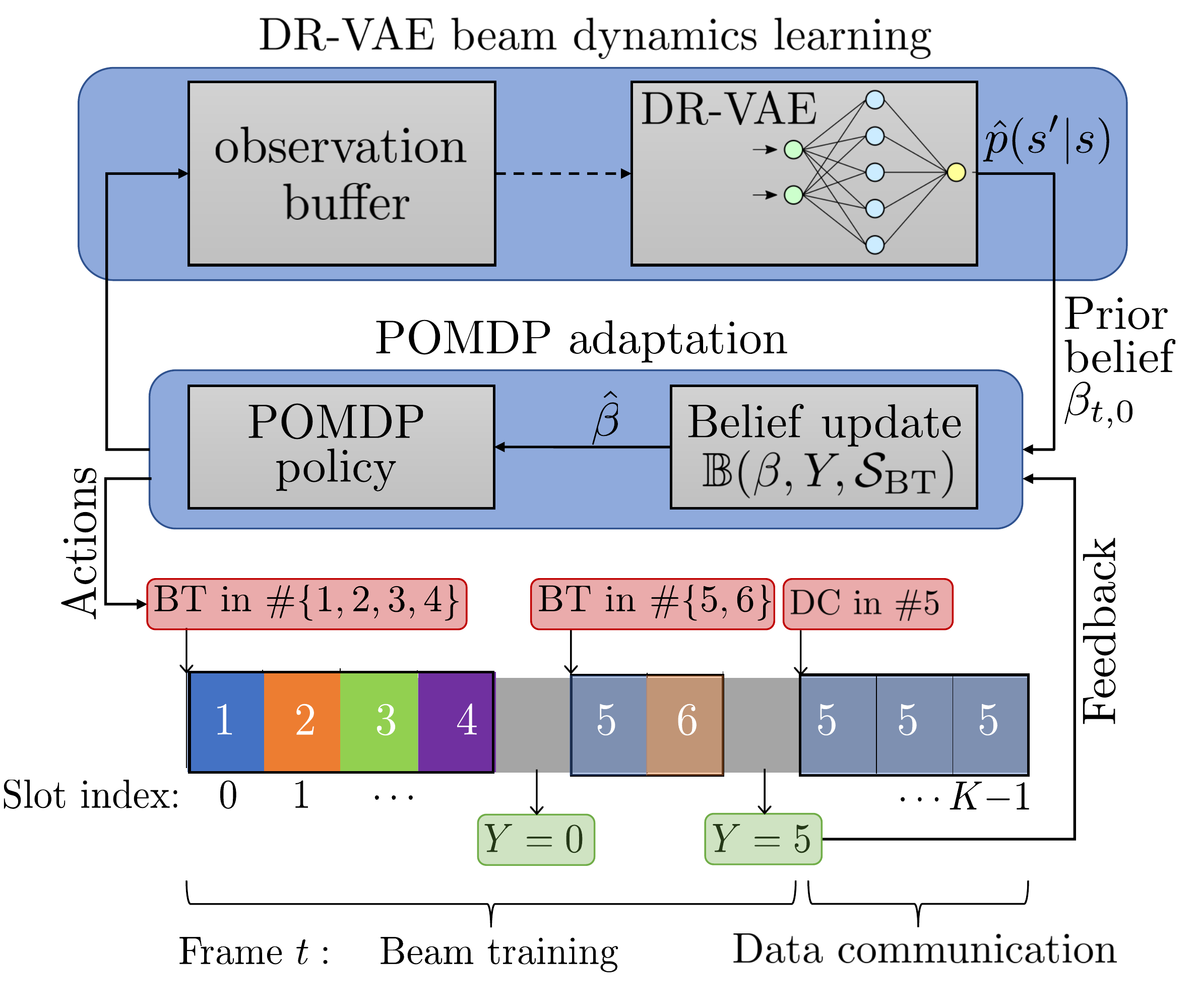} 
	\vspace{-2mm}
	\caption{
	Dual timescale framework (short- and long-timescale
	interactions shown with solid and dashed arrows, respectively) and example of POMDP policy, {based on \figref{fig:sys}, with the UE in the strongest beam pair index (BPI)  $5$:
	the BS executes BT on the set of BPIs $\{1,2,3,4\}$ over 4 consecutive slots, receives the feedback signal $Y{=}0$ (indicating the misalignment condition), then
	executes BT on the set of BPIs $\{5,6\}$ over two consecutive slots, receives the feedback signal $Y{=}5$ (indicating alignment on BPI 5); finally, it terminates BT and executes DC on the detected BPI $5$, until the end of the frame.}
	}\label{fig:two_timescale}
\end{figure}

  In this paper, we adopt {a multipath channel model with a LOS path \cite{TWC2019,blockage,alkhateeb_hybrid}, expressed as
\begin{align}
\label{eq:channel}
&\mathbf{H}_{t,k}{=}
\sum_{\ell=0}^{N_P}
\sqrt{M_{\rm tx} M_{\rm rx}} \cdot
h_{t,k}^{(\ell)}\cdot\Big(\mathbf{d}_{\rm rx}(\theta_t^{(\ell)})\cdot \mathbf{d}_{\rm tx}(\phi_t^{(\ell)})^{\rm H}\Big),
\end{align}
where $h_{t,k}^{(0)}{\sim}\mathcal{CN}(0,1/{\rm PL}_t)$ is the complex gain of the LOS path, with pathloss ${\rm PL}_t{=}(4 \pi d_t)^2/\lambda_c^2$,
 function of the UE-BS' distance $d_t$ and the carrier wavelength $\lambda_c$;
 $\theta_t^{(0)}$ and $\phi_t^{(0)}$ are the AoA and AoD of the LOS path (including both azimuth and elevation coordinates, since we use a 3D beamforming architecture), respectively;
 $\mathbf d_{\rm tx}(\cdot)$ and $\mathbf d_{\rm rx}(\cdot)$ are the unit norm  array response vectors of the BS' and UE's antenna arrays, respectively.
 The components indexed by $\ell{\geq}1$ denote $N_P$
 additional non-LOS (NLOS) paths, each with their own fading coefficient
 $h_{t,k}^{(\ell)}{\sim}\mathcal{CN}(0,\sigma_{\ell,t}^{2})$, AoA  $\theta_t^{(\ell)}$ and AoD $\phi_t^{(\ell)}$.
}


 {As observed in \cite{channel_model},} $d_t$, 
  {$\theta_t^{(\ell)}$, $\phi_t^{(\ell)}$ and $\sigma_{\ell,t}^{2}$
  are large-scale channel propagation features, hence
  remain (approximately)}
   constant during the frame duration, as reflected by the assumption $T_{\rm fr}{\leq}T_{b}$ discussed earlier.
 \label{page:longterm_bf}
 On the other hand, {the small-scale fading coefficients}  $h_{t,k}^{(\ell)}$ are fast-fading components, and are thus modeled as i.i.d over the slots within the frame. {In principle, the beamforming vectors $\mathbf f_{t,k}$ and $\mathbf c_{t,k}$ should be chosen to maximize the beamforming gain based on the instantaneous channel realization $\mathbf{H}_{t,k}$,
by solving $\max_{j\in\mathcal S}|\mathbf f^{(j)\rm H} \mathbf{H}_{t,k} \mathbf{c}^{(j)}|^2$. However, as also noted in \cite{channel_model}, doing so requires CSI at both transmitter and receiver, which may be unfeasible in large antenna systems. An alternative approach used in this paper is known as \emph{long-term beamforming} \cite{channel_model,4217735}, in which the beamforming vectors are optimized to the large-scale features but not the small-scale ones: the goal is to maximize the 
\emph{expected} beamforming gain $G_t$, averaged over the small-scale fading coefficients with the large-scale features kept constant, expressed as a function of the beamforming vectors $(\mathbf c,\mathbf f)$ as
 \begin{align*}
 \label{eq:Gtot}
 \numberthis
 G_t(\mathbf c,\mathbf f)
\triangleq
G_{t}^{\rm LOS}(\mathbf c,\mathbf f)+G_{t}^{\rm NLOS}(\mathbf c,\mathbf f),
 \end{align*} 
 where
 \begin{align*}
 \label{eq:GLOS}
 \numberthis
 G_{t}^{\rm LOS}(\mathbf c,\mathbf f)
 &\triangleq\frac{M_{\rm tx}M_{\rm rx}}{{\rm PL}_t}
 |\mathbf{d}_{\rm tx}(\phi_t^{(0)})^{\rm H}\mathbf{c}|^2
 |\mathbf{d}_{\rm rx}(\theta_t^{(0)})^{\rm H}\mathbf{f}|^2,\!\!
 \\
G_{t}^{\rm NLOS}(\mathbf c,\mathbf f)
& \triangleq
\!M_{\rm tx}M_{\rm rx}\!\!
\sum_{\ell=1}^{N_P}
\sigma_{\ell,t}^2
 |\mathbf{d}_{\rm tx}(\phi_t^{(\ell)})^{\rm H}\mathbf{c}|^2
 |\mathbf{d}_{\rm rx}(\theta_t^{(\ell)})^{\rm H}\mathbf{f}|^2
 \end{align*}
 are the expected gains along the
 LOS and NLOS paths.
 The maximum expected beamforming gain is then
 $G_t^{\max}{=}\max_{j\in\mathcal S}G_t(\mathbf c^{(j)}{,}\mathbf f^{(j)})$,
 maximized by the optimal BPI.}
 
 {Channel measurements \cite{channel_model} have shown that,
 even in \emph{dense urban environments},
 the mm-wave channel typically exhibits 1-3 paths, with the LOS one containing most of the signal energy, implying that
 $\sum_{\ell=1}^{N_P}\sigma_{\ell,t}^2{\ll}
1/{\rm PL}_t$. 
 Based on this fact, in this paper we assume that the maximum \emph{expected} beamforming gain is achieved along the LOS path, 
 \begin{align*}
 \label{Gmax}
G_t^{\max}
\approx
\max_{j\in\mathcal S} 
G_{t}^{\rm LOS}(\mathbf c^{(j)},\mathbf f^{(j)}),
\numberthis
 \end{align*}
 and is thus maximized by the strongest BPI (SBPI)
 \begin{align*}
 \label{Stdef}
S_t
=\argmax_{j\in\mathcal S} 
G_{t}^{\rm LOS}(\mathbf c^{(j)},\mathbf f^{(j)}),
\numberthis
 \end{align*}
 so that $(\mathbf c^{(S_t)},\mathbf f^{(S_t)})$ should be used at the BS and UE to maximize the expected beamforming gain in frame $t$.}
\vspace{-5mm}
\subsection{Sectored antenna {with binary SNR} model}
\label{secmodel}
We now approximate the beamforming gain using a \emph{sectored antenna model}~\cite{va2016beam, TWC2019, TVT2020}, which provides an analytically tractable yet valuable approximation of the actual beam pattern, as demonstrated numerically in \secref{sec:numres}. 
To develop this model, note that there is a relationship between the UE's position and the pathloss, AoA and AoD of the LOS path, denoted as
${\rm PL}(x)$,
$\theta(x)$ and $\phi(x)$ for a UE in position $x\in\mathcal X$ within the coverage area $\mathcal X$ of the BS.
From \eqref{eq:GLOS}, the expected beamforming gain along the LOS path is then expressed as
$$
G^{\rm LOS}(\mathbf c,\mathbf f,x)\triangleq\frac{M_{\rm tx}M_{\rm rx}}{{\rm PL}(x)}
 |\mathbf{d}_{\rm tx}(\phi(x))^{\rm H}\mathbf{c}|^2
 |\mathbf{d}_{\rm rx}(\theta(x))^{\rm H}\mathbf{f}|^2,
$$
for a UE in position $x$,
and the SBPI as
\begin{align}
\label{sbpi}
s^*(x) \triangleq \arg \max_{j \in {\mathcal S}} \
G^{\rm LOS}(\mathbf c^{(j)},\mathbf f^{(j)},x).
 \end{align}
With this definition, we partition the coverage area $\mathcal X$ into {irregularly shaped} sectors $\{\mathcal X_j\}_{j{\in}\mathcal S}$, with $\mathcal X_j{\triangleq}\{x{\in}\mathcal X{:}s^*(x){=}j\}$ representing
the set of positions in which the $j$th BPI
maximizes the expected beamforming gain (i.e., it is the SBPI).\footnote{{Note that $\mathcal X_j$ is empty, if the $j$th BPI does not maximize the expected beamforming gain within the coverage area of the BS (for instance, for beam pairs pointing away from each other). Since there are $|\mathcal S|$ BPIs, there are at most $|\mathcal S|$ such sectors, each one associated to a certain SBPI.}}

{Let $X_t$ be the position of the UE in frame $t$, so that $S_t{=}s^*(X_t)$ is the corresponding SBPI, and consider the $j$th BPI.}
 We denote the condition $X_t{\in}\mathcal X_j$ (or equivalently, $S_t{=}j$) as the \emph{beam-alignment condition} along BPI $j$: {in fact, when $S_t{=}j$, the $j$th
 BS' and UE's beamforming vector pair maximize
 the expected beamforming gain along the LOS path, as expressed by \eqref{Gmax} and \eqref{Stdef}}; we let
 ${\rm SNR}_{\rm BA}$ be the {expected SNR under the beam-alignment condition,
averaged over the small-scale fading coefficients.}
{This SNR requirement may be achieved via power control at the BS, to provide uniform signal quality over the coverage area. In fact, from the signal model \eqref{eq:signal_model} and using \eqref{eq:Gtot}, the expected SNR (averaged over  small-scale fading) in position $X_t{\in}\mathcal X_j$ under the $j$th BPI is expressed as
\begin{align}
\label{eq:expected_snr}
\!\!\!\mathbb E[{\rm SNR}_{t,k}]{=}
    \frac{P_{t,k}}{\sigma_w^2}G_t(\mathbf c^{(j)},\mathbf f^{(j)})
    {\geq}
    \frac{P_{t,k}}{\sigma_w^2}G^{\rm LOS}(\mathbf c^{(j)},\mathbf f^{(j)},X_t),\!\!\!
\end{align}
so that the condition $\mathbb E[{\rm SNR}_{t,k}]\geq {\rm SNR}_{\rm BA}$ over $\mathcal X_j$ under the $j$th BPI is guaranteed by
 choosing the BS' transmit power as
\begin{align}
\label{txpower}
P_{t,k}
=
\frac{\sigma_w^2\cdot{\rm SNR}_{\rm BA}}{\min_{x\in\mathcal X_j}
G^{\rm LOS}(\mathbf c^{(j)},\mathbf f^{(j)},x)}\triangleq P^{(j)}.
\end{align}}
Conversely, we denote $X_t{\notin}\mathcal X_j$ (i.e, $S_t{\neq}j$) as the \emph{beam-misalignment condition} along BPI $j$.
In this case, 
the expected beamforming gain is not maximized, yielding smaller expected SNR.
Since the beam-misalignment condition implies $G_t(\mathbf c^{(j)},\mathbf f^{(j)}){<} G_t^{\max}$,
we upper bound the {expected} SNR as $\frac{P^{(j)}}{\sigma_w^2}G_t(\mathbf c^{(j)},\mathbf f^{(j)}){\leq}
\rho{\cdot}{\rm SNR}_{\rm BA}$,
where $\rho{<}1$ is the misalignment to alignment gain ratio,
{which accounts for the effect of sidelobes and NLOS multipath.}

We have thus defined a binary SNR model, in which
the {expected} SNR is only a function of
the beam-alignment condition under the current choice of the BPI $j$:
under beam-alignment ($j{=}S_t$), the {expected} SNR is (at least) ${\rm SNR}_{\rm BA}$;
under beam-misalignment ($j{\neq}S_t$), the {expected} SNR is (at most) $\rho{\cdot}{\rm SNR}_{\rm BA}$; in this latter case, data communication is in outage since $\rho{\ll}1$.
{Note that it is desirable for ${\rm SNR}_{\rm BA}$ to be as large as possible, to maximize spectral efficiency; conversely, $\rho$ should be as small as possible to minimize errors in the detection of the SBPI (which 
explains the use of an upper bound on the expected SNR under beam-misalignment, as a worst case condition); additionally, the  \emph{power ratio} $\max_{j\in\mathcal S} P^{(j)}/\min_{j\in\mathcal S} P^{(j)}$ should be as small as possible to
minimize the peak-to-average power ratio at  the BS. These goals may be achieved via careful planning at deployment phase, including BS placement and beam design. We refer to Table \ref{table1} in \secref{sec:numres} for
an example of numerical values.}
\vspace{-5mm}
\subsection{Beam Training (BT) and Data Communication (DC)}
\label{BTDC}
We now introduce the BT and DC operations.
{The goal of the BT phase (of variable duration) is to detect the SBPI $S_t$, so that data communication can be performed with maximum expected beamforming gain (averaged over the small-scale fading) for the remainder of the frame.
 Both operations are executed within the frame duration, and the process is repeated in each frame, based on a POMDP  policy described in \secref{sec:POMDP}.}

\textbf{BT phase:}
During the BT phase, the BS selects and executes a sequence of BT actions.
A {single} BT action specifies a set of BPIs ${\mathcal S}_{\rm BT} {\subseteq} \mathcal S$ to be scanned, {selected by the POMDP controller (\secref{sec:POMDP}}); it is then  executed as follows:
\begin{enumerate}[leftmargin=*]
\item A sequence of beacon signals $\mathbf x$ is sent
{over the set of  BPIs ${\mathcal S}_{\rm BT}$, using one slot for each BPI in ${\mathcal S}_{\rm BT}$ (the scanning order is unimportant).}
Let $k_j$ be the slot during which BPI $j{\in}{\mathcal S}_{\rm BT}$ is scanned: in this slot,
the BS transmits with beamforming vector $\mathbf c^{(j)}$ and transmit power $P^{(j)}$, while the UE receives synchronously with combining vector $\mathbf f^{(j)}$;
the UE then processes the received signal $\mathbf{y}_{t,k_j}$ (see \eqref{eq:signal_model}) using a matched filter,
to estimate the SNR along BPI $j$ as
 \begin{align}
 \Gamma_{t}^{(j)}\triangleq\frac{|\mathbf{x}^{\rm H} \mathbf{y}_{t,k_j}|^2}{\sigma_w^2 \Vert\mathbf{x}\Vert_2^2}.
 \label{eq:Gamma}
\end{align}
\item After collecting the sequence $\{\Gamma_{t}^{(j)},\forall j{\in}{\mathcal S}_{\rm BT}\}$,
the UE 
transmits a discrete feedback signal $Y$ back to the BS in the next slot.
To generate $Y$, the UE first detects the strongest BPI $J{\triangleq} \argmax_{{j}\in{\mathcal S}_{\rm BT}}\Gamma_{t}^{(j)}$.
If $\Gamma_{t}^{(J)}{>}\eta$ (an SNR threshold $\eta$, {designed to trade off false-alarm and misdetection probabilities, as discussed next}),
then $Y{=}J$ indicates the index of the  SBPI detected; otherwise ($\Gamma_{t}^{(J)}{\leq}\eta$),
{$Y{=}0$} indicates detection of the beam-misalignment condition
over ${\mathcal S}_{\rm BT}$.
\end{enumerate}
{Since the BT action over the BPI set ${\mathcal S}_{\rm BT}$ contains
$|{\mathcal S}_{\rm BT}|$ BPIs, it takes $|{\mathcal S}_{\rm BT}|+1$ slots (including the
sequential beacons and feedback signal transmissions).}
Note that
the BT action and the feedback signal may be coordinated using low frequency control channels, with small overhead:
${\mathcal S}_{\rm BT}{\subseteq}\mathcal S$ can be encoded using  $|\mathcal S|$ bits, and
the feedback signal $Y{\in}\mathcal S{\cup}\{0\}$ can be encoded using at most $\log_2(|\mathcal S|+1)$ bits.

{We now discuss the feedback distribution} $\mathbb P_Y(y|s,{\mathcal S}_{\rm BT})\triangleq\mathbb P(Y=y|S_t=s,{\mathcal S}_{\rm BT})$, denoting  the probability of generating the feedback signal $Y=y$, given the ground truth SBPI $S_t$ and the 
BT action over the BPI set ${\mathcal S}_{\rm BT}$.
Its expression is provided in closed-form
 in \cite{TVT2020}.
Consider the case $s\in{\mathcal S}_{\rm BT}$ first, i.e., the ground truth SBPI is scanned during the BT action: we define
$$p_{\rm corr}\triangleq \mathbb P_Y(s|s,{\mathcal S}_{\rm BT}),\ p_{\rm md}\triangleq \mathbb P_Y(0|s,{\mathcal S}_{\rm BT})$$
as the  probabilities of correctly detecting the SBPI ($Y{=}s$),
and of incorrectly detecting beam-misalignment ($Y{=}0$). Then, the SBPI is incorrectly detected ($Y{\neq}s$ and $Y{\neq}0$)
with probability
 $$
 \mathbb P_Y(y|s,{\mathcal S}_{\rm BT})=\frac{1-p_{\rm corr}-p_{\rm md}}{|{\mathcal S}_{\rm BT}|-1},\ \forall y\in{\mathcal S}_{\rm BT}\setminus\{s\},
 $$
 i.e., detection errors are uniform over the remaining BPIs ${\mathcal S}_{\rm BT}{\setminus}\{s\}$ (this property follows  from the binary SNR model).
On the other hand, if $s{\notin}{\mathcal S}_{\rm BT}$, i.e., the ground truth SBPI is not scanned during the BT action, we
let $p_{\rm fa}$ be the false-alarm probability  that beam-alignment is detected, so that
\begin{align*}
\begin{cases}
\mathbb P_Y(y|s,{\mathcal S}_{\rm BT})=\frac{p_{\rm fa}}{|{\mathcal S}_{\rm BT}|},\ \forall y\in{\mathcal S}_{\rm BT},\\
  \mathbb P_Y(0|s,{\mathcal S}_{\rm BT})=1-p_{\rm fa},
  \end{cases}
\end{align*}
i.e., false-alarm errors are uniform across the set of BPIs ${\mathcal S}_{\rm BT}$ scanned during BT
(as a result of the binary SNR model).

{The threshold $\eta$ trades off false-alarm and misdetection probabilities: if $S_t{\notin}{\mathcal S}_{\rm BT}$, then 
 beam-misalignment is \emph{correctly} detected ($Y{=}0$)
if and only if $\Gamma_{t}^{(j)}{<}\eta,\forall j{\in}{\mathcal S}_{\rm BT}$, 
yielding
\begin{align*}
&p_{\rm fa}=
1-\mathbb P(\Gamma_{t}^{(j)}<\eta,\forall j\in{\mathcal S}_{\rm BT}|S_t\notin{\mathcal S}_{\rm BT})
\\
&=
1-
\Big(1-e^{-\frac{\eta}{1+\rho\cdot{\rm SNR}_{\rm BA}L_{\rm sy}}}\Big)
^{|{\mathcal S}_{\rm BT}|}, \numberthis
\label{eq:fa}
\end{align*}
where, using the binary SNR model of \secref{secmodel}, we used the fact that $\Gamma_{t}^{(j)}{\sim}{\mathcal E}(1{+}\rho{\cdot}{\rm SNR}_{\rm BA} L_{\rm sy}),\forall j{\neq} S_t$, independently across $j$.
On the other hand, if
$S_t\in{\mathcal S}_{\rm BT}$ (the SBPI $S_t$ is scanned), then 
the misalignment condition $Y{=}0$ is \emph{incorrectly} detected
if and only if $\Gamma_{t}^{(j)}{<}\eta,\forall j\in{\mathcal S}_{\rm BT}$, 
yielding
\begin{align*}
\label{eq:md}
\numberthis
&p_{\rm md}
=
\mathbb P(\Gamma_{t}^{(j)}<\eta,\forall j\in{\mathcal S}_{\rm BT}|S_t\in{\mathcal S}_{\rm BT})
\\
&=
\Big(1-e^{-\frac{\eta}{1+{\rm SNR}_{\rm BA}L_{\rm sy}}}\Big)\cdot
\Big(1-e^{-\frac{\eta}{1+\rho\cdot{\rm SNR}_{\rm BA}L_{\rm sy}}}\Big)
^{|{\mathcal S}_{\rm BT}|-1},
\end{align*}
where we used
 $\Gamma_{t}^{(s)}{\sim}{\mathcal E}(1{+}{\rm SNR}_{\rm BA} L_{\rm sy})$ along the SBPI $s{=}S_t$.
In this paper, we set $\eta$
so that $p_{\rm fa}{=}p_{\rm md}$, determined via the bisection method, by leveraging the facts that 
$p_{\rm fa}$ and $p_{\rm md}$
are \emph{increasing} and \emph{decreasing} functions of $\eta$, respectively. However, the following results hold for a generic choice of~$\eta$.
\label{page:eta}
}\\
\indent \textbf{DC phase:} 
At the start of the DC phase,
{in slot $k_{\rm DC}$ within the frame,}
the BS selects
 a BPI  $s_{\rm DC}$  (both $k_{\rm DC}$ and $s_{\rm DC}$ are selected by the POMDP controller of \secref{sec:POMDP}).
Data communication then occurs 
{for the remaining $K-k_{\rm DC}$ slots}
until the end of the frame,
with a fixed rate $R$
using the selected BPI:
the BS transmits with beamforming vector $\mathbf c^{(s_{\rm DC})}$
and transmit power $P^{(s_{\rm DC})}$ given by \eqref{txpower},
while the UE receives with combining vector $\mathbf f^{(s_{\rm DC})}$.
If $S_t \neq s_{\rm DC}$, then the
beam-misalignment condition occurs, and  communication is in outage; otherwise
(beam-alignment condition $S_t{=}s_{\rm DC}$),
outage occurs if
 {$R{>}W_{\rm tot}\log_2 (1+{\rm SNR}_{t,k})$,}
 where 
 {${\rm SNR}_{t,k}$
is the instantaneous SNR, function of the small-scale fading conditions,}
 $W_{\rm tot}$ is the bandwidth.
 Owing to the binary SNR model of \secref{secmodel}, it follows that
 ${\rm SNR}_{t,k}\sim\mathcal E({\rm SNR}_{\rm BA})$ under the beam-alignment condition,
 yielding
 the success probability
 with respect to the small-scale fading conditions
{\begin{align}
\label{psucc}
\mathbb P_{\rm succ}{=}
\mathbb P\Bigr(
{\rm SNR}_{t,k}{>}2^{R/W_{\rm tot}}{-}1
\Bigr)
{=}e^{{-}\frac{2^{R/W_{\rm tot}}{-}1}{{\rm SNR}_{\rm BA}}}.\end{align}}
{We define the expected frame spectral efficiency, averaged over the small-scale fading, as a function of $k_{\rm DC}$, $s_{\rm DC}$
 and the \emph{unknown} ground truth SBPI
 $S_t{=}s$,
 as
\begin{align}
\label{frth}
{\rm SE}_{\rm fr}(k_{\rm DC},s_{\rm DC};s)=
{\rm SE}_{\rm BA}{\cdot}
\Big(1{-}\frac{k_{\rm DC}}{K}\Big)\cdot \mathbb I[s= s_{\rm DC}],
\end{align}
where
${\rm SE}_{\rm BA}{\triangleq}\frac{\mathbb P_{\rm succ}R}{W_{\rm tot}}$ is the expected spectral efficiency under the maximum beamforming condition,
 $1{-}\frac{k_{\rm DC}}{K}$ is the loss due to the BT overhead,
 and $\mathbb I[s{=}s_{\rm DC}]$ 
 captures the outage condition under beam-misalignment.}
\vspace{-4mm}
\subsection{SBPI dynamics}
 \label{sec:mobility}
 The mobility of the UE
 induces temporally correlated dynamics on the SBPI $S_t$, which we exploit to reduce the BT overhead via predictive beam-tracking.
 {In principle, the sequence $\{S_t,t{\geq}0\}$ may not satisfy the Markov property (for instance, a vehicle moving at constant speed along a road would spend a deterministic amount of time on each SBPI traversed along the path). In this paper, for analytical and computational tractability,}
we model $\{S_t,t{\geq}0\}$ as Markovian.
 We will demonstrate numerically in \secref{sec:numres} that this assumption yields a good approximation of non-Markovian dynamics.
 {In addition, we assume that the associated Markov chain is time-homogeneous, so that the transition model does not change with time: in fact, long-term mobility patterns typically vary slowly compared to the timescales of the communication scenario under consideration.}
 Let 
 $p(s^{\prime}|s)\triangleq\mathbb P(S_{t+1}{=}s^{\prime}|S_t {=} s), \forall s{\in}\mathcal S, \forall s^{\prime}{\in}\mathcal S \cup \{ \bar s\}$
 be the one-frame transition probability from the current SBPI $s$ to the next SBPI $s^{\prime}$;
the additional state $\bar s$
 indicates that the UE exited the BS coverage area and can no longer be served by it.
 
 In practice, $p(\cdot|\cdot)$
 is unknown and needs to be estimated via BT feedback -- a departure from \cite{stss} and our previous work  \cite{TVT2020}, which assumed  prior knowledge of $p(\cdot|\cdot)$. 
A \emph{naive} approach
estimates
$p(\cdot|\cdot)$ based on
 the sequence
 $\{(\hat s_t,\hat s_{t+1}),t{\in}\mathcal T\}$
 of SBPIs' transitions detected via BT (e.g., exhaustive search) at time frames $t{\in}\mathcal T$, as
  \begin{align}
  \label{eq:naive}
   \hat p(s^{\prime}|s) = \frac{ \sum_{t \in \mathcal T} \mathbb I[\hat s_t = s, \hat s_{t+1} = s^{\prime}]}{\sum_{t \in \mathcal T} \mathbb I[\hat s_t = s]},\ \forall s,s^{\prime}\in\mathcal S.
   \end{align}
Yet, detection errors caused by noise, beam imperfections,
   and {NLOS multipath}
may   degrade the estimation quality, hence the performance of adaptive BT schemes based on it.
 The design of an estimation procedure robust to these errors
is developed in \secref{sec:vae} using a novel technique based on a deep recurrent variational autoencoder (DR-VAE) architecture.
\vspace{-3mm}
\section{Short Timescale: adaptive BT via PBVI}
\label{sec:POMDP}
This section introduces the POMDP adaptation and proposes an efficient PBVI algorithm to determine an approximately optimal BT policy.
Since the short-timescale adaptation is the same for all frames, we do not express the dependence on the frame index $t$, whenever unambiguous.
{The goal is to design a policy $\pi$
 that dictates the sequence of BT/DC actions, so as to
 to maximize the expected frame spectral efficiency:
\begin{align}
\label{pomdpgoal}
\!\!\textbf{[MAX-SE]:}\ \max_{\pi}
\mathbb E_{\pi}\Big[{\rm SE}_{\rm fr}(k_{\rm DC},s_{\rm DC};S)\Big|S{\sim}\bfbeta_{0} \Big],\!
\end{align}
where $\bfbeta_{0}(s){=}\mathbb P(S{=}s)$ is
the probability mass function of the SBPI $S$ at the start of the frame, provided by the DR-VAE (\secref{sec:vae});
the expectation is with respect to $S{\sim}\bfbeta_{0}$, the sequence of actions and observations
dictated by policy $\pi$.}
  {Since the state $S$ is unknown, we formulate this optimization as a POMDP, with} individual components defined as follows.
 
\textbf{Time horizon}:
the slots within the frame $\mathcal K{=}\{0,{\cdots},K{-}1\}$.

\textbf{State:} the SBPI $S{\in}\mathcal S$,
taking value from the state space $\mathcal S$; $S$
 remains constant during the frame duration, but may change from one frame to the next according to $p(\cdot|\cdot)$ (see \secref{sec:mobility}).


\textbf{Action and reward models:}
In slot $k$ {within the frame}, the BS selects whether to perform BT or DC, {and on which BPIs}.
If the DC action is selected {on the BPI $s_{\rm DC}{\in}\mathcal S$,
then data communication occurs over the BPI $s_{\rm DC}$ for the remainder of the frame, as described in \secref{BTDC}},
{and the decision period within the current frame terminates}.
{With the goal stated in \textbf{[MAX-SE]},} the reward under the DC action is {${\rm SE}_{\rm fr}(k,s_{\rm DC};S)$ as in \eqref{frth}.}
{On the other hand, if the BT action is selected over the set of BPIs ${\mathcal S}_{\rm BT}$, then the BS proceeds with a round of BT, {which occupies $|{\mathcal S}_{\rm BT}|{+}1$ slots
for the sequential beacon and feedback signal transmissions} (see \secref{BTDC}). The next BT or DC decision is then taken in slot $k{+}|{\mathcal S}_{\rm BT}|{+}1$ of the current frame.}
Since a BT action must
{terminate before the end of the frame}, the BT action space in slot $k$ is
$\mathcal A_{{\rm BT},k}\equiv \{{\mathcal S}_{\rm BT}{\subseteq}\mathcal S{:}|{\mathcal S}_{\rm BT}|\leq K{-}1{-}k\}$.
 {Consistently with
 the optimization problem \textbf{[MAX-SE]}, a BT action accrues no reward
 since it does not generate data.}

\textbf{Observation model:}
After executing a BT action over the BPI set ${\mathcal S}_{\rm BT}$, the BS observes the feedback signal
$Y$ from the observation space
$ \mathcal Y \triangleq {\mathcal S}_{\rm BT} \cup \{ 0\}$ as described in \secref{BTDC}.
{We assume that the DC action does not generate observations, since these may require additional feedback signaling and involve upper layer protocols, such as acknowledgements of data packets. However, if available, these may be included in the belief update at the end of the frame.}


\textbf{Belief Update:}   
Since $S{\in}\mathcal S$ is unknown, we use the belief $\bfbeta_k$ as a POMDP state in slot $k$, i.e.,
the probability distribution over $S$, given the history of actions and observations 
up to slot $k$ (not included),
 taking value from the $|\mathcal S|$-dimensional probability simplex $\mathcal B$  (initially, $\bfbeta_{0}$ is the prior belief provided by the DR-VAE). 
 {In fact, $\bfbeta_k$ is a sufficient statistic for decision-making in POMDPs \cite{krishnamurthy2016partially}.}
 If a BT action over the BPI set ${\mathcal S}_{\rm BT}$ is selected in slot $k$,
 {it is executed in the following $|{\mathcal S}_{\rm BT}|{+}1$  slots} and it generates the feedback signal $Y{=}y$;
 the new belief
is then updated
after the action termination (in slot $k{+}|{\mathcal S}_{\rm BT}|{+}1$)
using Bayes' rule as
\begin{align}
\label{eq:POMDP_belief_update}
\bfbeta_{k{+}|{\mathcal S}_{\rm BT}|{+}1}(s) =
 \frac{\bfbeta_{k}(s) \mathbb P_Y(y|s,{\mathcal S}_{\rm BT})}{\sum_{j \in \mathcal S}\bfbeta_{k}(j) \mathbb P_Y(y|j,{\mathcal S}_{\rm BT})},\ \forall s{\in}\mathcal S,
\end{align}
where $\mathbb P_Y(y|j,{\mathcal S}_{\rm BT})$ is the feedback distribution given in \secref{BTDC}.
We write the belief update compactly as $\bfbeta_{k+|{\mathcal S}_{\rm BT}|+1}{=}\mathbb B(\bfbeta_k,y,{\mathcal S}_{\rm BT})$.
On the other hand, if a DC action is selected in slot $k$, it is executed until the end of the frame
and generates no feedback, so that $\bfbeta_{K}{=}\bfbeta_{k}$.
At the end of frame $t$, with $\bfbeta_{K}^{(t)}$ computed with the procedure described above, the prior belief for the next frame is computed based on the transition model $\hat p(\cdot |\cdot )$
learned by the DR-VAE,
as
\begin{align}
\label{eq:prior_up}
\bfbeta_{0}^{(t+1)}(s^{\prime}) = \frac{\sum_{s \in \mathcal S} \hat p(s^{\prime}|s) \bfbeta_{K}^{(t)}(s)}{\sum_{(s,s^{\prime\prime}) \in \mathcal S^2} \hat p(s^{\prime\prime}|s) \bfbeta_{K}^{(t)}(s)}, \forall s^{\prime} \in \mathcal S,
\end{align}
so that the POMDP policy may be carried out in the next frame using $\bfbeta_{0}^{(t+1)}$ as a prior belief, and so on.

Accurate estimation of $p(\cdot |\cdot )$ is critical to achieve good performance: errors in $\hat p(\cdot |\cdot )$ may lead to inaccurate predictions of the SBPI, resulting in increased BT overhead and decreased spectral efficiency, as shown numerically in \secref{sec:numres}.
To address this challenge, in \secref{sec:vae} we propose a state-of-the-art estimation of the SBPI's dynamics based on DR-VAE.

\label{page:suff_pol}
\textbf{Policy:}
a mapping from the current belief $\bfbeta_k$ to a DC or BT action, denoted as $\pi_k$ in slot $k$.
{In fact, mapping beliefs to actions is sufficient to achieve optimality in POMDPs \cite{krishnamurthy2016partially}.}

With the POMDP thus defined, its optimization
may be carried out via value iteration \cite{Pineau2006PointbasedAF}.
 Define the optimal value function in slot
$K$ under belief $\bfbeta$ as $V_K^{*}(\bfbeta){=}0$ (since
no further data can be transmitted, once the frame terminates); then,
the optimal value function in slot
 $k{=}K-1,\dots,0$ under belief $\bfbeta$ 
 can be computed recursively as
\begin{align}
\label{eq:val_iter}
&V_k^*(\bfbeta){=}
\max\Big\{\overbrace{{\rm SE}_{\rm BA}{\cdot}\Big(1-\frac{k}{K}\Big)\cdot\max_{s_{\rm DC}\in\mathcal S}\bfbeta(s_{\rm DC})}^{\triangleq V_{k}^{(\mathrm{DC})}(\bfbeta)},
\\&
\max_{{\mathcal S}_{\rm BT}\in\mathcal A_{{\rm BT},k}}
\!
\overbrace{
\sum_{s,y}
\bfbeta(s)
\mathbb P_Y(y|s,{\mathcal S}_{\rm BT})
V_{k{+}|{\mathcal S}_{\rm BT}|{+}1}^*(\mathbb B(\bfbeta,y,{\mathcal S}_{\rm BT}))}^{\triangleq
V_{k}^{(\mathrm{BT})}(\bfbeta,{\mathcal S}_{\rm BT})}
\!\!\Big\}\!,\!\!
\nonumber
\end{align}
yielding the optimal decision between DC or BT action (maximizer of the outer $\max$),
the optimal DC action (maximizer of the first inner $\max$, resulting in the value function $V_{k}^{(\mathrm{DC})}(\bfbeta)$)
 and the optimal BT set (maximizer of the second inner $\max$,
 where $V_{k}^{(\mathrm{BT})}(\bfbeta,{\mathcal S}_{\rm BT})$ is the value function associated to the BT set ${\mathcal S}_{\rm BT}$). In fact, if a DC action is selected, then the reward is 
${\rm SE}_{\rm BA}{\cdot}(1-k/K)$ until the end of the frame, as long as beam-alignment is achieved,
whose chances are maximized by selecting the most likely BPI ($\argmax_{s_{\rm DC}\in\mathcal S}\bfbeta(s_{\rm DC})$).
Conversely, if the BT action over the BPI set ${\mathcal S}_{\rm BT}$ is selected,
no reward is collected; since its duration is $|{\mathcal S}_{\rm BT}|{+}1$,
the future value function is taken at time $k{+}|{\mathcal S}_{\rm BT}|{+}1$, and the belief is updated
via the map $\mathbb B$ of \eqref{eq:POMDP_belief_update},
based on the feedback signal; 
maximization of $V_{k}^{(\mathrm{BT})}(\bfbeta,{\mathcal S}_{\rm BT})$ yields the optimal BT action.

Since the value function is {piecewise linear convex} \cite{Pineau2006PointbasedAF}, it can be expressed using a finite set of $|\mathcal S|$-dimensional hyperplanes from a properly defined set
 $\mathcal Q_{k}\subset \mathbb R^{|\mathcal S|}$:\footnote{With a slight abuse of notation, we refer to the vectors $\bfalpha\in\mathcal Q_k$ as \emph{hyperplanes} throughout this manuscript, meaning that
 they are vectors of coefficients that generate the hyperplanes $\inprod{\bfbeta}{\bfalpha},\bfbeta\in\mathbb R^{|\mathcal S|}$.
 }
$$V_k^*(\bfbeta) = \max_{\bfalpha \in \mathcal Q_{k}} \inprod{\bfbeta }{\bfalpha},$$
with  $\mathcal Q_{k}$ computed recursively as
$\mathcal Q_K{=}\{\mathbf 0\}$ 
and,
for 
 $k{<}K$,
\begin{align}
\label{eq:Q_val_iter}
&\mathcal Q_{k}{\equiv}
\cup_{i=1}^{|\mathcal S|}\!\!\Big\{{\rm SE}_{\rm BA}{\cdot}\Big(1-\frac{k}{K}\Big)\cdot\mathbf e_i \Big\}
\bigcup
 \cup_{{\mathcal S}_{\rm BT}\in\mathcal A_{{\rm BT},k}}
\\&
 \Big\{
\sum_{y \in \mathcal Y} \mathbb P_Y(y|\cdot,{\mathcal S}_{\rm BT}) \odot \bfalpha^{(y)}{:}
[\bfalpha^{(y)}]_{y \in \mathcal Y} {\in} \mathcal Q_{k+|{\mathcal S}_{\rm BT}|+1}^{|\mathcal Y|}
  \Big\}.\nonumber
 \end{align}
 In \eqref{eq:Q_val_iter}, the hyperplane ${\rm SE}_{\rm BA}(1{-}k/K)\mathbf e_i$ corresponds to DC action $i{\in} \mathcal S$;
 $\sum_{y \in \mathcal Y}  \mathbb P_Y(y|\cdot,{\mathcal S}_{\rm BT}){\odot}\bfalpha^{(y)}$ corresponds to BT action over the BPI set ${\mathcal S}_{\rm BT}{\in}\mathcal A_{{\rm BT},k}$, where $\bfalpha^{(y)}$ is the hyperplane corresponding to the future value function in slot $k{+}|{\mathcal S}_{\rm BT}|{+}1$, reached after observing $y$.
 The union of such hyperplanes correspond to all possible current and future
actions that may be taken after observing $y$.
  
In exact value iteration, $\mathcal Q_{k}$ grows 
doubly exponentially with iteration $k$,
 an intractable problem for any reasonably sized task.
  To address this challenge:
 1) we present structural properties in Theorem \ref{strucPOMDP}, which enable a compact belief representation, reduced computational complexity and memory requirements;
 2) in \secref{PBVI}, we present a
 PBVI
  approximate solver for POMDPs \cite{Pineau2006PointbasedAF}, that exploits these structural properties to enable an efficient implementation.
 
 To introduce the structural properties, let $\mathcal P(\bfalpha)$ be the set containing $\bfalpha$ and all permutations of its elements,
so that $|\mathcal P(\bfalpha)|{\leq}|\mathcal S|!$ for $\bfalpha{\in}\mathbb R^{|\mathcal S|}$.
Let $\mathrm{sort}(\mathbf a)$ be the vector obtained by sorting the elements of $\mathbf a$ in non-increasing order, 
so that $\tilde{\mathbf a}(1){\geq} \tilde{\mathbf a}(2){\dots}\tilde{\mathbf a}(|\mathcal S|)$ for $\tilde{\mathbf a}{=}\mathrm{sort}(\mathbf a)$.
Let $\mathcal Q_{k}^{\rm sort}\equiv\{\mathrm{sort}(\bfalpha):\bfalpha{\in}\mathcal Q_k\}$ be the set containing only the sorted elements of $\mathcal Q_k$.
 \begin{theorem}
 \label{strucPOMDP}
 We have the following properties:
\begin{description}
\item[P1:] If $\bfalpha\in\mathcal Q_k$, then $\mathcal P(\bfalpha)\subseteq\mathcal Q_k$;
\item[P2:] 
$V_k^*(\bfbeta'){=}V_k^*(\bfbeta){=}\max\limits_{\bfalpha \in\mathcal Q_{k}^{\rm sort}} \inprod{\mathrm{sort}(\bfbeta)}{\bfalpha},\ \forall \bfbeta'{\in}\mathcal P(\bfbeta),\bfbeta{\in}\mathcal B.$
\end{description}
 \end{theorem}
 \begin{proof}
See Appendix A.
 \end{proof}
\vspace{-3mm}
P1 states that $\mathcal Q_k$ contains all the permutations of its hyperplanes, and directly follows from the
{binary SNR model of \secref{secmodel}, which induces}
symmetries in the observation, reward and action models. 
P2 {is a direct consequence of P1, and} states that permutations of a belief yield the same value function: {intuitively, two beliefs that are permutations of each other carry the same amount of uncertainty on the current SBPI, hence yield the same value.}
{The practical implications of P1 and P2 are
 in terms of reduced computational complexity and memory requirements of the POMDP optimization: thanks to P2,}
 the value iteration algorithm can be restricted to the set of sorted beliefs $\mathcal B^{\rm sort}{=}\{\mathrm{sort}(\bfbeta){:}\bfbeta{\in}\mathcal B\}$, since the value of all other beliefs can be inferred from this restricted set;
thanks to P1, it is sufficient to store the set of sorted hyperplanes $\mathcal Q_{k}^{\rm sort}$,
since $\mathcal Q_k$ can be generated through all the permutations of the elements in $\mathcal Q_{k}^{\rm sort}$,
representing a computational and memory saving by a factor ${\sim}|\mathcal S|!$.
P1 and P2 will be exploited in the next section to develop a PBVI solver.
\vspace{-2mm}
\subsection{Point-Based Value Iteration (PBVI)}
\label{PBVI}
The key idea behind PBVI is to
restrict the value iteration step in \eqref{eq:val_iter}
 to a finite set of belief points $ \tilde{\mathcal B}{\subset}\mathcal B$, chosen as representative of the entire belief-space $\mathcal B$.
 To this end,  PBVI constructs recursively a $|\tilde{\mathcal B}|$-dimensional set of hyperplanes $\tilde{\mathcal Q}_k$,
 each with its associated action.
 From the structural properties of Theorem \ref{strucPOMDP}, we can restrict
 the belief and hyperplane sets to contain only sorted elements without loss of optimality, and we label them as
  $ \tilde{\mathcal B}^{\rm sort}{\subset}\mathcal B^{\rm sort}$
and $\tilde{\mathcal Q}_{k}^{\rm sort}$. 
These sorted sets virtually represent much larger  ones given by all permutations of their elements.
With the set of hyperplanes defined,
the value function
 can then be approximated 
 by restricting the maximization in Theorem \ref{strucPOMDP}.P2 to the set $\tilde{\mathcal Q}_{k}^{\rm sort}$, yielding
\begin{align}
\label{approxVfun}
\tilde V_k(\bfbeta)=
\max_{\bfalpha \in\tilde{\mathcal Q}_{k}^{\rm sort}} \inprod{\mathrm{sort}(\bfbeta)}{\bfalpha},\ \forall \bfbeta\in\mathcal B;
\end{align}
{the (approximately) optimal action under the belief $\bfbeta$ is then obtained
using a two-steps procedure:}
\begin{enumerate}[leftmargin=*]
    \item {The optimal action for the \emph{sorted} belief
$\mathrm{sort}(\bfbeta)$ is found as the one associated with the optimizing hyperplane in \eqref{approxVfun};}
\item {this action then needs to be mapped back to the \emph{unsorted} belief $\bfbeta$; to do so,
since $\bfbeta$ 
is obtained by a suitable permutation of the elements of $\mathrm{sort}(\bfbeta)$, the same permutation is applied to the action computed in the first step, which yields the optimal action for the unsorted belief $\bfbeta$.}
\end{enumerate}
\begin{algorithm2e}[t]
\DontPrintSemicolon
\SetNoFillComment
\SetKwFunction{Union}{Union}\SetKwFunction{FindCompress}{FindCompress}
\SetKwInOut{Input}{input}\SetKwInOut{Output}{output}
\caption{PBVI optimization algorithm}
\label{alg:PBVI_main}
\Input{Belief set $\tilde{\mathcal B}^{\rm sort}$}
\textbf{init:} $\tilde{\mathcal Q}_{K}^{\rm sort}= \{ \boldsymbol 0\} $\;
\For{$k=K-1,\cdots,0$}{
\For{$\forall \bfbeta \in \tilde{\mathcal B}^{\rm sort}$}{
$\tilde V_{k}^{(\mathrm{DC})}(\bfbeta)=\inprod{\bfbeta}{{\rm SE}_{\rm BA}{\cdot}\left(1-\frac{k}{K}\right)\mathbf e_1}$ (value of DC action)\;
\For{${\mathcal S}_{\rm BT}\in\mathcal A_{{\rm BT},k}$ and $y\in{\mathcal S}_{\rm BT}\cup\{0\}$}{
$\tilde V_{k+|{\mathcal S}_{\rm BT}|+1}(\mathbb B(\bfbeta,y,{\mathcal S}_{\rm BT}))
=\max_{\bfalpha\in\tilde{\mathcal Q}_{k+|{\mathcal S}_{\rm BT}|+1}^{\rm sort}}\inprod{\mathrm{sort}(\mathbb B(\bfbeta,y,{\mathcal S}_{\rm BT}))}{\bfalpha}$
and maximizer $\bfalpha_{{\mathcal S}_{\rm BT}}^{(y)}$\;
$
\tilde V_{k}^{(\mathrm{BT})}(\bfbeta,{\mathcal S}_{\rm BT})=\sum_{s,y} \bfbeta(s)$
$
{\times}{\mathbb P_Y(y|s,{\mathcal S}_{\rm BT})}
\tilde V_{k+|{\mathcal S}_{\rm BT}|+1}(\mathbb B(\bfbeta,y,{\mathcal S}_{\rm BT}))$
}
$\!\!\!{\mathcal S}_{{\rm BT},k}^*{=}\!\!\argmax\limits_{{\mathcal S}_{\rm BT}\in\mathcal A_{{\rm BT},k}}\!\!\tilde V_k^{(\mathrm{BT})}\!\!(\bfbeta,{\mathcal S}_{\rm BT})$ (best BT~action)\;
 \If{$\tilde V_{k}^{(\mathrm{DC})}(\bfbeta)\geq\tilde V_{k}^{(\mathrm{BT})}(\bfbeta,{\mathcal S}_{{\rm BT},k}^*)$ \label{line:add_hyp_DC}}  
{
$\tilde{\mathcal Q}_{k}^{\rm sort}\ni {\rm SE}_{\rm BA}{\cdot}\left(1-\frac{k}{K}\right)\mathbf e_1$ (DC is optimal)\; 
}
\label{line:add_hyp_BT}\Else(BT over the set ${\mathcal S}_{{\rm BT},k}^*$ is optimal){
$\tilde{\mathcal Q}_{k}^{\rm sort}{\ni}
\mathrm{sort}(\sum_{y{\in}\mathcal Y} \mathbb P_Y(y|\cdot,{\mathcal S}_{{\rm BT},k}^*){\odot}\bfalpha_{{\mathcal S}_{{\rm BT},k}^*}^{(y)})$
 \label{line:hypend}
}
}
}
\Return {$\{\tilde{\mathcal Q}_{k}^{\rm sort}{:}k{\in}\mathcal K\}$ and associated optimal actions.}
\end{algorithm2e}
  The set of hyperplanes and corresponding actions
 is determined with Algorithm \ref{alg:PBVI_main}, 
 similar to  \eqref{eq:Q_val_iter} with two key differences: 1) 
 since we restrict PBVI to the set $\tilde{\mathcal B}^{\rm sort}$,
only the hyperplanes that maximize the value function are preserved;
  2) we exploit the structural properties of Theorem \ref{strucPOMDP} for an efficient implementation, as detailed next.
 Starting from $\tilde{\mathcal Q}_{K}^{\rm sort}{=}\{\boldsymbol 0\}$ (step 1),
 the algorithm proceeds backward in time (step 2)
 to compute $\tilde{\mathcal Q}_{k}^{\rm sort}$ from previously computed $\tilde{\mathcal Q}_{j}^{\rm sort},j{>}k$ (steps 3-12),
 by iterating through all belief points $\bfbeta{\in}\tilde {\mathcal B}^{\rm sort}$ (step 3).
 To determine the optimal hyperplane associated to a certain $\bfbeta$, 
 the associated value function $\tilde V_k(\bfbeta)$ is computed similarly to \eqref{eq:val_iter}:
 in step 4, the value $\tilde V_{k}^{(\mathrm{DC})}(\bfbeta)$ of a DC action is computed (since $\bfbeta$ is sorted, 
 $\inprod{\bfbeta}{\mathbf e_1}{=}
 \max_{s}\bfbeta(s)$, yielding $V_{k}^{(\mathrm{DC})}(\bfbeta)$ as in \eqref{eq:val_iter});
 then, the value of each BT action over the BPI set ${\mathcal S}_{\rm BT}$ is computed by
 1) calculating the future value function
 as in \eqref{approxVfun}, for each possible observation $y$ (step 6);
 and 2) taking the expectation under the belief $\bfbeta$ to compute $\tilde V_{k}^{(\mathrm{BT})}(\bfbeta,{\mathcal S}_{\rm BT})$ (step 7);
 in step 8, the optimal BT action is found by maximizing 
 $\tilde V_{k}^{(\mathrm{BT})}(\bfbeta,{\mathcal S}_{\rm BT})$.
 With the values of the optimal DC action ($\tilde V^{(\mathrm{DC})}$) and optimal BT action ($\tilde V^{(\mathrm{BT})}$) thus determined, 
 if $\tilde V^{(\mathrm{DC})}{\geq}\tilde V^{(\mathrm{BT})}$, then the DC action is optimal, and the associated hyperplane is stored in
 $\tilde{\mathcal Q}_{k}^{\rm sort}$ (step 10).
 Otherwise, the BT action over the BPI set ${\mathcal S}_{{\rm BT},k}^*$ is optimal, and the 
 associated hyperplane is computed in step 12, similar to \eqref{eq:Q_val_iter}, sorted and then stored in  $\tilde{\mathcal Q}_{k}^{\rm sort}$.
 The algorithm continues until $k{=}0$.
 Since at most one hyperplane is added to
  $\tilde{\mathcal Q}_{k}^{\rm sort}$ for each $\bfbeta{\in}\tilde {\mathcal B}^{\rm sort}$, it follows $|\tilde{\mathcal Q}_{k}^{\rm sort}|{\leq}|\tilde{\mathcal B}^{\rm sort}|$, yielding a linear-time value iteration algorithm.

With the set of hyperplanes computed, the policy is executed in each frame as in Algorithm~\ref{alg:PBVI_adapt}, starting from slot $0$:
in slot $k$, given the belief $\bfbeta_{k}$, the optimal action is
{determined using the two-steps procedure described after \eqref{approxVfun}} (step 1);
if a DC action is selected (the one that maximizes the belief, steps 2-3), it is executed until the end of the frame and the reward is accrued;
otherwise, the optimal BT action is executed, the feedback $Y$ is collected, the belief is updated via Bayes' rule, and the process is repeated in slot $k{+}|{\mathcal S}_{\rm BT}|{+}1$ (steps 5-6).
\vspace{-5mm}
\subsection{Low-Complexity Error-Robust MDP-based Policy Design}
\label{sec:MDP}
Although PBVI finds an approximately optimal policy, it may incur high computational cost, especially with high dimensional belief spaces. To overcome it, in this section we propose an MDP policy based on the assumption of error-free feedback.
This case can be cast as a special case with
$p_{\rm corr}{=}1$, $p_{\rm md}{=}p_{\rm fa}{=}0$,
yielding 
$\mathbb P_Y(y|s,{\mathcal S}_{\rm BT}){=}\mathbb I[y{=}s],\ \forall s\in{\mathcal S}_{\rm BT}$ and
$\mathbb P_Y(y|s,{\mathcal S}_{\rm BT}){=}\mathbb I[y=0],\forall s\notin{\mathcal S}_{\rm BT}$.
Without loss  of generality, we assume $\bfbeta_0{\in}\mathcal B^{\rm sort}$ (i.e., it is sorted, see Theorem~\ref{strucPOMDP}).
  The belief update  \eqref{eq:POMDP_belief_update}
under the BT action over the BPI set ${\mathcal S}_{\rm BT}$ after observing $Y{\in}{\mathcal S}_{\rm BT}{\cup}\{0\}$
 is then specialized as 
$\bfbeta_{k+|{\mathcal S}_{\rm BT}|+1}(s) {=}\frac{\bfbeta_{k}(s)}{\sum_{j \notin {\mathcal S}_{\rm BT}}\bfbeta_{k}(j)}\mathbb I[s{\notin}{\mathcal S}_{\rm BT}]$ for $Y{=}0$ and
$\bfbeta_{k{+}|{\mathcal S}_{\rm BT}|{+}1}(s){=}\mathbb I[s=Y]$ for $Y\in{\mathcal S}_{\rm BT}$, 
i.e., the SBPI is revealed correctly when $s\in{\mathcal S}_{\rm BT}$.
Under such belief updates, it can be shown by induction that $\bfbeta_{k}$ can be expressed as a function
of its support $\mathcal U_{k}$ and the prior belief $\bfbeta_{0}$ as
\begin{align}
\label{bk}
    \bfbeta_{k}(s) = \frac{\bfbeta_{0}(s)}{\sum_{j\in \mathcal U_{k}} \bfbeta_{0}(j)} \mathbb I[s\in \mathcal U_{k}], \forall s \in \mathcal S,
\end{align}
with support updated recursively after observing $Y$ as
\begin{align}
\label{eq:supp_update}
\mathcal U_{k+|{\mathcal S}_{\rm BT}|+1}= \begin{cases}
\mathcal U_{k}\setminus {\mathcal S}_{\rm BT}, & Y = 0,\\
\{Y\}, & Y\in{\mathcal S}_{\rm BT}.
\end{cases}
\end{align}
Hence, given $\bfbeta_{0}$, the support $\mathcal U_{k}$ is a sufficient statistic for decision-making.
Using the structure of the belief and the error-free observation model, the value iteration algorithm,
expressed as a function of $\mathcal U$, then specializes as
\begin{align*}
& 
V_k^*(\mathcal U)=
  \max \Bigg\{
  \overbrace{{\rm SE}_{\rm BA}{\cdot}\Big(1-\frac{k}{K}\Big)\frac{\max_{s\in\mathcal U}\bfbeta_0(s)}{\sum_{s\in\mathcal U}\bfbeta_0(s)}}^{\triangleq V_{k}^{(\mathrm{DC})}(\mathcal U)}
,
\max_{{\mathcal S}_{\rm BT}\in\mathcal A_{{\rm BT},k}}
\\&
\overbrace{
\frac{
\sum\limits_{
s\in\mathcal U\cap{\mathcal S}_{\rm BT}}
\!\!\!\!\!\!
\bfbeta_0(s)
V_{k{+}|{\mathcal S}_{\rm BT}|{+}1}^*\!(\{s\})
{+}\!\!\!\!\!\!\!\!\sum\limits_{s\in\mathcal U\setminus{\mathcal S}_{\rm BT}}\!\!\!\!
\!\!\!\!\bfbeta_0(s)V_{k{+}|{\mathcal S}_{\rm BT}|{+}1}^*\!(\mathcal U{\setminus}{\mathcal S}_{\rm BT})
}{\sum_{s\in\mathcal U}\bfbeta_0(s)}}^{\triangleq V_{k}^{(\mathrm{BT})}(\mathcal U,{\mathcal S}_{\rm BT})}\!\!\Bigg\}\!.
\nonumber
  \end{align*}
  \begin{algorithm2e}[t]
\DontPrintSemicolon
\SetNoFillComment
\SetKwFunction{Union}{Union}\SetKwFunction{FindCompress}{FindCompress}
\SetKwInOut{Input}{input}\SetKwInOut{Output}{output}
\caption{PBVI-based adaptation algorithm}
\label{alg:PBVI_adapt}
\Input{$\{\tilde{\mathcal Q}_{k}^{\rm sort}: k \in \mathcal K\}$ and associated optimal actions; initial belief $\bfbeta_{0}$; initialize $k=0$;}
In slot $k$, given $\bfbeta_{k}$: solve \eqref{approxVfun};
 {determine optimal action (two-steps procedure described after \eqref{approxVfun});}\!\!\!\!\!\;
\If{DC action $s_{\rm DC}$ selected}{execute the DC action
$s_{\rm DC}=\argmax_{s\in\mathcal S}\bfbeta_k(s)$; accrue reward
${\rm SE}_{\rm BA}{\cdot}\left(1-\frac{k}{K}\right)\mathbb I[S_t=s_{\rm DC}]$;
{\bf terminate.}
}
\Else(\emph{BT action over the BPI set ${\mathcal S}_{\rm BT}$ selected})
{
Execute the action; observe $Y\sim \mathbb P_Y(\cdot|S,{\mathcal S}_{\rm BT})$ at the end of slot $k{+}|{\mathcal S}_{\rm BT}|$;\;
Update the belief $\bfbeta_{k{+}|{\mathcal S}_{\rm BT}|{+}1}{=}\mathbb B(\bfbeta_{k},Y,{\mathcal S}_{\rm BT})$ as in \eqref{eq:POMDP_belief_update}; continue from step 1 with $k\gets k{+}|{\mathcal S}_{\rm BT}|{+}1$.
}
\end{algorithm2e}

We now present structural results showing that, among other auxiliary properties,
the optimal BT action should scan the most likely BPIs in $\mathcal U$.
We will then use these properties to further simplify the value iteration algorithm.
Note that this result may not hold for the general POMDP model since the uncertainty cannot be completely removed after BT.
 \begin{theorem}
 \label{thm:MDP}
Let $\mathcal U$ be the current support in slot $k$. Then,
 \begin{description}
\item[P1:] If $\mathcal U\equiv\{s\}$, the DC action $s_{\rm DC}{=}s$ is optimal, with value
  $V_{k}^*(\{s\})={\rm SE}_{\rm BA}{\cdot}\left(1-\frac{k}{K}\right)$;
  \item[P2:] $V_k^*(\mathcal U)\geq V_{k+1}^*(\mathcal U),\ \forall k,\forall\mathcal U$ (monotonicity);
  \item[P3:] $\argmax_{{\mathcal S}_{\rm BT}\in\mathcal A_{{\rm BT},k}}V_{k}^{(\mathrm{BT})}(\mathcal U,{\mathcal S}_{\rm BT})\subseteq\mathcal U$;
\item[P4:] 
The optimal BT set scans the most likely BPIs from  $\mathcal U$.
 \end{description}
 \end{theorem}
\begin{proof}
See Appendix B.
\end{proof}
\vspace{-2mm}
The implications of the above theorem are twofold: 1) once the SBPI is detected, it is optimal to switch to DC until the end of the frame (P1); 2) the belief support $\mathcal U_k$ takes the form $\mathcal U_k{=}\{u_k,u_{k}{+}1,\dots,|\mathcal S|\}$, for some $u_k{\in}\mathcal S$ so that the index $u_k$ is a 
sufficient statistic to represent $\mathcal U_k$ (note that we assume $\bfbeta_0$ is already sorted). 
This can be seen by induction: initially,
$\mathcal U_0{=}\mathcal S$, so that $u_0{=}1$;
now, assume $\mathcal U_k=\{u_k,u_{k}+1,\dots,|\mathcal S|\}$ for some $k\geq 0$:\footnote{Since $\bfbeta_0$ is sorted, using \eqref{bk} with $\mathcal U_k{=}\{u_k,u_{k}{+}1,\cdots,|\mathcal S|\}$, it follows that
$\bfbeta_k(u_k){\geq}\bfbeta_k(u_k{+}1){\geq}{\cdots}\bfbeta_k(|\mathcal S|)$
and $\bfbeta_k(s){=}0,\forall s<u_k$.}
if a DC action is taken, then the most likely BPI is selected ($u_k$), and DC occurs until the end of the frame;
otherwise, if the BT action
${\mathcal S}_{\rm BT}$ is selected, then P3-P4 dictate 
to scan the most likely BPIs, so that
$ {\mathcal S}_{\rm BT}{=}\{u_k,u_{k}{+}1,\dots,u_{k}{+}n-1\}$,
for some $n$ (the determination of $n$ is discussed after \eqref{MDPsimple}).
If, after  execution of the BT action over the BPI set ${\mathcal S}_{\rm BT}$, the SBPI $s$ is detected, then the 
system switches to DC until the end of the frame, with value $V_{k+n+1}^*(\{s\})$ (P1); otherwise, the new support becomes 
$\mathcal U_{k+n+1}{\equiv}\mathcal U_k{\setminus}\hat {\mathcal S}{\equiv}\{u_k{+}n,u_{k}{+}n{+}1,\dots,|\mathcal S|\}$,
hence $u_{k+n+1}=u_k{+}n$, which proves the induction.
Using the sufficient statistic $u{\in}\mathcal S$, the value function simplifies to
\begin{align}
\label{MDPsimple}
  &V_k^*(u) =
  \max \Bigg\{
 {\rm SE}_{\rm BA}{\cdot}\Big(1-\frac{k}{K}\Big)\frac{\bfbeta_0(u)}{\sum_{s=u}^{|\mathcal S|}\bfbeta_0(s)}
,\\&
\max_{n\in\mathcal N_k(u)}
\overbrace{\frac{\sum_{s=u}^{u+n-1}\bfbeta_0(s)}{\sum_{s=u}^{|\mathcal S|}\bfbeta_0(s)}{\rm SE}_{\rm BA}{\cdot}\Big(1-\frac{k+n+1}{K}\Big)}^{\triangleq V_{k}^{(\mathrm{BT})}(u,n)}
\nonumber
\\&
\qquad\qquad\overbrace{+\frac{\sum_{s=u+n}^{|\mathcal S|}\bfbeta_0(s)}{\sum_{s=u}^{|\mathcal S|}\bfbeta_0(s)}V_{k+n+1}^*(u+n)}^{V_{k}^{(\mathrm{BT})}(u,n)\text{ \ (continued)}}\qquad\ \ \ \Bigg\},\nonumber
  \end{align}
  where $\mathcal N_k(u)=\{1,\dots,\min\{K-k-1,|\mathcal S|-u+1\}\}$ is the set of feasible BT set sizes.
  Note that the maximizer
  $n^*{=}\argmax_{n\in\mathcal N_k(u)}V_{k}^{(\mathrm{BT})}(u,n)$
 yields the optimal number of BPIs that should be scanned during BT, so that the
 optimal BT set is ${\mathcal S}_{\rm BT}^*{=}\{u_k,u_{k}+1,\dots,u_{k}+n^*-1\}$.
 
 Since the MDP-based policy neglects BT feedback errors, it may perform poorly when they occur. 
We incorporate error robustness with an error-robust MDP (ER-MDP) that executes the
policy based on the MDP formulation defined above, but updates the belief based on the POMDP model, which takes into account the statistics of BT errors.
In other words, the MDP policy defines the number $n^*$ of BPIs to scan; the $n^*$ most likely BPIs are then scanned during BT,
the feedback $Y$ is collected, the belief is updated as in \eqref{eq:POMDP_belief_update}, and so on.

Both POMDP- and (ER-)MDP-based policies require the prior belief at the start of each frame to select actions. Next, we
 propose a DR-VAE-based learning framework to learn
 the transition model $p(\cdot|\cdot)$ used in \eqref{eq:prior_up} to update the prior belief.
 \vspace{-3mm}
\section{Long Timescale: Learning Beam Dynamics via Deep Recurrent 
Variational Autoencoder}
\label{sec:vae}
We now present the learning module, aiming to learn the SBPI transition model {from $N{\geq}1$ \emph{episodes}. Each episode corresponds to a certain UE served by the BS,} and is expressed as
a sequence of BT actions
and associated BT feedback collected over the episode duration, {corresponding to the time interval during which the UE stays within the coverage area of the BS}.
{Consider a single episode $N{=}1$ (later generalized to $N{\geq}1$),} occupying 
$T{+}1$ frames.
Let $M_t$ be the total number of BT actions in frame $t$ (a random variable, due to the decision-making process),
$\mathcal S_{{\rm BT,}m}^{(t)}$
be the $m$th BT action of frame $t$,
generated by following an arbitrary policy $\pi$,
and  $Y_m^{(t)}$ be its associated observation, {generated based on the model of \secref{BTDC}}.
Let
$\mathcal H_{m}^{(t)}$, with
$m{=}1,{\dots},M_t$ and $t{=}0,{\dots},T$,
be
the history of BT actions and observations collected up to the $m$th BT action and observation of frame $t$. Recursively, 
$\mathcal H_{m}^{(0)}{=}(\mathcal S_{{\rm BT,}i}^{(0)},Y_i^{(0)})_{i=1,\dots,m}$
and
$\mathcal H_{m}^{(t)}=[
\mathcal H_{M_{t-1}}^{(t-1)},
(\mathcal S_{{\rm BT,}i}^{(t)},Y_i^{(t)})_{i=0,\dots,m}]$.

 Let 
$ {\mathcal P}_{\boldpsi}(s_{0:T}){\triangleq} \bfbeta_{0}^{(0)}(s_0)\prod_{t=1}^{T} p_{\boldpsi}(s_t|s_{t-1})$ be the joint probability of the state sequence $S_{0:T}{=}s_{0:T}$, where $\bfbeta_{0}^{(0)}$ is the prior belief over $S_0$ at $t{=}0$ (e.g., uniform) and $p_\boldpsi$ is the unknown transition model, parameterized by $\boldpsi$.
 Then, we aim to maximize the marginal likelihood of the BT action-observation sequences collected over the episode, $f(\mathcal H_{M_T}^{(T)}|\boldpsi)$, 
 with respect to the parameter $\boldpsi$, stated as\footnote{We use $f(\cdot)$ to denote suitable distributions, as clear from the context.}
\begin{align*}
\max_{\boldpsi} f(\mathcal H_{M_T}^{(T)}|\boldpsi) 
&=\max_{\boldpsi} \sum_{s_{0:T}} 
{\mathcal P}_{\boldpsi}(s_{0:T}) 
f(\mathcal H_{M_T}^{(T)}|s_{0:T}). \numberthis
\label{eq:marginal}
 \end{align*} 
 Yes, this problem is intractable, due to the lack of closed-form and the marginalization over the latent variables. 
Approximation techniques have been proposed to overcome this challenge, 
by employing a surrogate metric instead of \eqref{eq:marginal}. These include expectation-maximization (EM)-based algorithms such as Baum-Welch \cite{baum_welch}, 
which perform an alternating optimization of a non-convex variational objective,
and variational techniques such as the variational autoencoder (VAE) \cite{VAE},
which jointly learn separate posterior and prior state transition models. Thanks to the joint optimization procedure, variational techniques typically outperform EM-based techniques, as verified numerically in \secref{sec:numres}, and will be adopted in this paper.

 VAE is one of the most powerful tools to learn latent variable models \cite{VAE}.
 It comprises two coupled but independently parameterized models:  the \emph{encoder}, which provides the posterior distribution over the latent states given the observations; the \emph{decoder}, which measures the representation quality of the latent states produced by the encoder via the observation model and the prior distribution of the latent state variable, thus forcing the encoder to learn a meaningful representation of the latent states from the observations. The deep recurrent (DR-) VAE  used in this paper extends the VAE to temporally correlated observations, in our case, obtained through the sampling of a POMDP \cite{rvae}.
 The main idea of DR-VAE is to learn two models: a posterior model
 $Q_{\boldnu}(S_{0:T}|\mathcal H_{M_T}^{(T)})$,
parameterized by $\boldnu$, used to infer the state sequence
 $S_{0:T}$ from the history of BT actions and observations;
 the prior model $\mathcal P_{\boldpsi}(S_{0:T})$, parameterized by $\boldpsi$.
 Given $\mathcal H_{M_T}^{(T)}$, the parameter vectors $(\boldnu,\boldpsi)$ are optimized by maximizing the \emph{evidence lower bound} (ELBO), {defined for a single episode as}
 \begin{align*}
 &{\rm ELBO}(\boldnu,\boldpsi|\mathcal H_{M_T}^{(T)})
 {\triangleq}\mathbb E_{Q_{\boldnu}}\!\Bigg[\!\ln\frac{ \mathcal P_{\boldpsi}(S_{0:T}) f(\mathcal H_{M_T}^{(T)}|S_{0:T})}{Q_{\boldnu}(S_{0:T}|\mathcal H_{M_T}^{(T)})}
\Bigg|\mathcal H_{M_T}^{(T)}
 \Bigg]\!,
\end{align*}
where the expectation is computed with respect to the
{state sequence generated through}
the posterior model $Q_{\boldnu}$,
so that the optimization (restricted to a single episode) is stated as
\begin{align}
\label{ELBOmax}
\max_{\boldnu,\boldpsi}\ {\rm ELBO}(\boldnu,\boldpsi|\mathcal H_{M_T}^{(T)}).
\end{align}
Using Jensen's inequality $\mathbb E[\ln(\cdot)]{\leq}\ln(\mathbb E[\cdot])$, it can be shown
\begin{align*}
{\rm ELBO}(\boldnu,\boldpsi|\mathcal H_{M_T}^{(T)})
 & \leq  \ln  \left(\mathbb E_{\mathcal P_\boldpsi} \left[ f(\mathcal H_{M_T}^{(T)}|S_{0:T})\right]\right) \\&\triangleq \ln(f(\mathcal H_{M_T}^{(T)}|\boldpsi)),
\end{align*}
i.e., ELBO is a lower bound to the log-likelihood function that we originally aimed to maximize in \eqref{eq:marginal}.
We now proceed to simplifying the ELBO.
First,
note that, using the product law of conditional probability
and the definition of $\mathcal H_{m}^{(t)}$,
  \begin{align*} 
 \label{eq:joint_obs_act}
 \numberthis
&f(\mathcal H_{M_T}^{(T)}|S_{0:T})
=
 \prod_{t=1}^{T} 
  \prod_{m=1}^{M_t} f(
  \mathcal S_{{\rm BT},m}^{(t)},
  Y_{m}^{(t)}
|S_{0:T},
\mathcal H_{m-1}^{(t)}
),\\&
=\prod_{t=1}^{T}
  \underbrace{\prod_{m=1}^{M_t} \mathbb P_Y(Y_{m}^{(t)}|S_{t},\mathcal S_{{\rm BT},m}^{(t)})}_{\triangleq\exp\{\mathcal L_t(S_t) \}}\prod_{m=1}^{M_t}f(\mathcal S_{{\rm BT},m}^{(t)}|\mathcal H_{m-1}^{(t)}
),
 \end{align*}
where in the last step we used two facts:
 1) given $(S_t,\mathcal S_{{\rm BT},m}^{(t)})$, $Y_{m}^{(t)}$ is independent of
 future states and of
  the past; 2) $\mathcal S_{{\rm BT},m}^{(t)}$ is obtained from policy $\pi$ -- a function of the history of BT actions and observations only.
  We have also defined the sum of log-likelihoods in frame $t$ as
  $\mathcal L_t(s) \triangleq \sum_{m=1}^{M_t}
\ln \mathbb P_Y(Y_{m}^{(t)}|s,\mathcal S_{{\rm BT},m}^{(t)})$, {which can be computed based on the observation model}.
   Next, we express the posterior model using the product law of conditional probability as
  $$
  Q_{\boldnu}(S_{0:T}|\mathcal H_{M_T}^{(T)})=Q_{\boldnu}(S_{0}|\mathcal H_{M_T}^{(T)})\prod_{t=1}^TQ_{\boldnu}(S_t|S_{0:t-1},\mathcal H_{M_T}^{(T)}).
  $$
  For tractability,
in DR-VAE settings, the following structure for $Q_{\boldnu}$ is used \cite{rvae}:
\begin{align}
\nonumber
&Q_{\boldnu}(S_{0}|\mathcal H_{M_T}^{(T)})={\bfbeta}_{{\rm post}}^{(0)}(S_0),\\&
Q_{\boldnu}(S_t|S_{0:t-1},\mathcal H_{M_T}^{(T)})
=q_{\boldnu}(S_t|S_{t-1},\mathbfcal L_t),
\label{eq:post}
\end{align}
where $\mathbfcal{L}_t{\triangleq}[\mathcal L_t(s)]_{s\in \mathcal S}$;
 ${\bfbeta}_{{\rm post}}^{(0)}$ is the posterior belief over $S_0$ after collecting the BT actions and observations in frame $0$,
\begin{align}
\label{betapost}
 {\bfbeta}_{{\rm post}}^{(0)}(s) = \frac{ \exp\left\{\mathcal L_0(s)\right\}{\bfbeta}_{0}^{(0)}(s)}{\sum_{\tilde s \in \mathcal S} \exp\left\{\mathcal L_0(\tilde s)\right\} {\bfbeta}_{0}^{(0)}(\tilde s)}.
 \end{align}  
The structure in \eqref{eq:post} captures the intuitive fact that knowledge of
$S_{0:t-2}$ and $\mathcal H_{M_T}^{(T)}$ is not informative to infer $S_t$, when $(S_{t-1},\mathbfcal L_t)$ is given.
In other words, $S_t$ can be inferred from the previous state $S_{t-1}$ (through the transition model) and the log-likelihood of the observations ($\mathbfcal L_t$).
Using \eqref{eq:joint_obs_act}, we can then express
  the optimization  in \eqref{ELBOmax} in the equivalent form
 \begin{align}
\label{ELBOmax2}
\max_{\boldnu,\boldpsi}
\ \underbrace{\mathbb E_{Q_{\boldnu}} \Bigg[
\sum_{t=1}^T\mathcal L_t(S_t)-\ln\frac{q_{\boldnu}(S_t|S_{t-1},\mathbfcal L_t)}{p_{\boldpsi}(S_t|S_{t-1})}
\Bigg|\mathbfcal L_{0:T}
 \Bigg]}_{\triangleq \widehat{\rm ELBO}(\boldnu,\boldpsi|\mathbfcal L_{0:T})},
\end{align}
where we neglected the terms
independent of the optimization variables $(\boldnu,\boldpsi)$: $f(\mathcal S_{{\rm BT},m}^{(t)}|\mathcal H_{m-1}^{(t)}
)$,
 $\bfbeta_0^{(0)}$ and ${\bfbeta}_{{\rm post}}^{(0)}$.
\\\indent
{By extending the previous analysis to $N{\geq}1$ episodes,
with the $n$th episode
of duration $T^{(n)}{+}1$,
expressed by the log-likelihood sequence  $\mathbfcal L_{0:T^{(n)}}^{(n)}$},
 the overall design is carried out by maximizing the ELBO averaged over the $N$ episodes as
 \begin{align}
 \label{overElBO}
\textbf{[LEARNING]:}\   \max_{\boldnu,\boldpsi}\  
\underbrace{
\frac{1}{N} \sum_{n=1}^{N}
\widehat{\rm ELBO}(\boldnu,\boldpsi|\mathbfcal L_{0:T^{(n)}}^{(n)})
}_{\triangleq \overline {\rm ELBO}(\boldnu,\boldpsi)}
,
  \end{align} 
  which will be the focus of Secs.~\ref{encdecparam} and \ref{sec:RVAEopt}.
\vspace{-5mm}
\subsection{Encoder and Decoder parameterization}
\label{encdecparam}
{For gradient based learning,} 
we encode states using the one-hot encoding, i.e., the $s$th standard basis column vector $\mathbf e_s$ denotes the SBPI $s{\in}\mathcal S$. {We will then use a continuous approximation of the one-hot encoded state in the backpropagation step of the training algorithm.}
We denote the one-hot encoded state in boldface as $\mathbf s_t{=}\mathbf e_{S_t}$, to distinguish it from $S_t{\in}\mathcal S$.

\textbf{Encoder:}
 The encoder models the posterior transition from $S_{t-1}$ to $S_t$ after observing $\mathbfcal L_t$.
Like previous work \cite{rvae}, we choose $\ln q_{\boldnu}( S_t| S_{t-1},\mathbfcal L_t)$ to be a recurrent neural network with weights and biases denoted by $\boldnu$.\footnote{For convenience, we express probabilities in the log domain.} The output of the neural network is produced by the softmax activation.
Consistently with the one-hot encoding of the state, we organize the elements in the matrix 
$[\mathbf {lnQ}_{\boldnu}(\mathbfcal L_t)]_{s',s}=\ln q_{\boldnu}(s'|s,\mathbfcal L_t)$,
 so that $\ln q_{\boldnu}(s'|s,\mathbfcal L_t)=\mathbf e_{s'}^{\top}{\cdot}\mathbf{lnQ}_{\boldnu}(\mathbfcal L_t){\cdot}\mathbf e_s$.

 \textbf{Decoder:}
 The decoder (generative model) represents the
joint distribution of state and observation sequences through the
 log-likelihood $\mathcal L_t(S_t)$ and the SBPI transition distribution $p_{\boldpsi}( S_t| S_{t-1})$, parametrized by $\boldpsi$.
Consistently with the one-hot encoding of the state, we organize the elements in the matrix $[\mathbf{lnP}_{\boldpsi}]_{s',s}{=}\ln p_{\boldpsi}(s'|s)$,
 so that $\ln p_{\boldpsi}(s'|s){=}\mathbf e_{s'}^\top{\cdot}\mathbf{lnP}_{\boldpsi}{\cdot}\mathbf e_s$.
  Since neural networks are universal function approximators and are well suited to gradient-based learning, we choose  $\mathbf{lnP}_{\boldpsi}$ to be a feedforward neural network with learnable parameters $\boldpsi$. The output of the neural network is produced by the softmax activation.
  
  {With this parameterization, the expectation in \eqref{ELBOmax2} can then be conveniently expressed as
  \begin{align}
  \label{1hotELBO}
  &\widehat{\rm ELBO}(\boldnu,\boldpsi|\mathbfcal L_{0:T})
 {=}
  \mathbb E_{Q_{\boldnu}}\Bigg[\sum_{t=1}^T  z_{t}(\boldnu,\boldpsi)\Bigg|\mathbfcal L_{0:T}\Bigg],
\end{align}
where
\begin{align}
\label{zt}
z_{t}(\boldnu,\boldpsi)\triangleq
\mathbf s_t^\top{\cdot}\Big(\mathbfcal{L}_t
{-}\mathbf{lnQ}_{\boldnu}(\mathbfcal L_t){\cdot}\mathbf s_{t-1}
{+}\mathbf{lnP}_{\boldpsi}{\cdot}\mathbf s_{t-1}\Big).
\end{align}
  }
 The optimization details are discussed in the next section. 
\vspace{-4mm}
\subsection{Optimization Algorithm}
\label{sec:RVAEopt}
The DR-VAE autoencoder is trained
{by solving the optimization problem \eqref{overElBO}}
via stochastic gradient ascent (SGA).
{To do so, each iteration of SGA is composed of two steps:
a \emph{forward propagation}, in which
$N_{\rm trg}$
independent trajectories of $\mathbf s_{0:T}$ (one-hot encoded)
are generated by the encoder, using the current posterior model $Q_\boldnu$ (since the expectation that defines \eqref{overElBO} via \eqref{ELBOmax2} is based on $Q_\boldnu$);
a \emph{backward propagation}, which aims to estimate
the gradient of $\overline {\rm ELBO}(\boldnu,\boldpsi)$ in \eqref{overElBO} based on the generated sequences.}
However,
{while gradients with respect to $\boldpsi$ can be done straightforwardly via \eqref{1hotELBO},}
gradient calculations with respect to $\boldnu$ are not tractable since the expectation in  \eqref{ELBOmax2} is taken over the latent variables $\mathbf s_{0:T}{\sim}Q_{\boldnu}$, whose joint distribution depends on $\boldnu$: {hence, $\mathbf s_t$ in \eqref{1hotELBO} is a stochastic function of $\boldnu$.} In the VAE literature, latent variable reparameterization techniques are proposed to
{decouple the expectation from such parametric dependence}.
 This is achieved by sampling {a random vector
 ${\bf R}_t{\in}\mathbb R^{|\mathcal S|}$ from a suitable distribution $f_R$,}
 and by defining a function $\mathbf g_\boldnu(\cdot;\mathbfcal L_t,\mathbf s_{t-1}):{\mathbb R^{|\mathcal S|}}\mapsto\{\mathbf e_i:i\in\mathcal S\}$
 such that 
 \begin{align}
 \label{gumbel}
 q_{\boldnu}(s|S_{t-1},\mathbfcal L_t){=}\mathbb P(\mathbf g_\boldnu({\bf R}_t;\mathbfcal L_t,\mathbf s_{t-1}){=}\mathbf e_s|\mathbfcal L_t,\mathbf s_{t-1}),\forall s.
 \end{align}
In other words,
 {$f_R$ and the function $g_\boldnu$ are designed so as to}
  generate states with the same statistics as $q_{\boldnu}$.
 The expectation $\mathbb E_{Q_{\boldnu}}$ in \eqref{1hotELBO}
 can then be replaced with an expectation $\mathbb E_{R}$ with respect to i.i.d. random variables
 $\{{\bf R}_t:t\geq 0\}$ {independent of $\boldnu$}, with $\mathbf s_t$ generated recursively as $\mathbf s_t=\mathbf g_\boldnu({\bf R}_t;\mathbfcal L_t,\mathbf s_{t-1})$,
 which enables tractable estimation of stochastic gradient estimates of the ELBO. 
We use the Gumbel-max reparameterization technique for this purpose,
proposed in \cite{Gumbel} for non-recurrent VAEs;
 to the best of our knowledge, this is the first paper to adopt it with DR-VAEs.
 Specifically, this techniques generates {a sequence of random vectors $\{{\bf R}_t{\in}\mathbb R^{|\mathcal S|}:t{\geq} 0\}$} from the standard Gumbel distribution,
with $\mathbf R_{t,i}{\sim}\mathcal G, \forall i \in\mathcal S$, i.i.d. across $i$ and $t$, and computes
 \begin{align}
\label{eq:s_gumbel}
&\mathbf s_0=\mathbf g(\mathbf R_0;\mathbfcal L_0)
{\triangleq}\argmax_{\mathbf e_{i}:i\in\mathcal S} \left[\mathbf R_{0,i}{+}
\ln {\bfbeta}_{{\rm post}}^{(0)}(i)\right],\\&
\mathbf s_t{=}\mathbf g_\boldnu(\mathbf R_t;\mathbfcal L_t,\mathbf s_{t-1})
{\triangleq}\argmax_{\mathbf e_i:i\in\mathcal S}\!\!\left[\mathbf R_{t,i}{+}
\mathbf e_{i}^{\top}\!{\cdot}\mathbf{lnQ}_{\boldnu}(\mathbfcal L_t){\cdot}\mathbf s_{t-1}
\right],
\nonumber
\end{align}
(one can show that \eqref{gumbel} holds, using the fact that $-\ln(E){\sim}\mathcal G$ if $E{\sim}\mathcal E(1)$).
\begin{algorithm2e}[t]
\DontPrintSemicolon
\SetNoFillComment
\SetKwFunction{Union}{Union}\SetKwFunction{FindCompress}{FindCompress}
\SetKwInOut{Input}{input}\SetKwInOut{Output}{output}
\caption[Caption for LOF]{DR-VAE training via SGA\footnotemark}
\label{alg:VAE_train}
\Input {initialize $(\boldnu,\boldpsi)=(\boldnu_0,\boldpsi_0)$,
${\rm epoch}=1$}
Sample batch of $N$ episodes $\{\mathbfcal L_{0:T^{(n)}}^{(n)}:n=1,\dots,N\}$ following adaptation policies $\{\pi^{(n)}:n=1,\dots,N\}$\;
\For{each episode $n=1,\cdots,N$}{
\For{$m=1,\cdots,N_{\rm trg}$}{ \label{line:grad1}
Sample $\mathbf R_{0}{\sim}\mathcal G^{|\mathcal S|}$; $\mathbf s_0{=}\mathbf g(\mathbf R_0{;}\mathbfcal L_0)$
via \eqref{eq:s_gumbel}, \eqref{betapost}
\label{line:s_0} \;
{\bf Forward propagation}:
\For {$t = 1,2,\cdots, T$ }{
Sample $\mathbf R_{t}\sim \mathcal G^{|\mathcal S|}$;
$\mathbf s_t=
\mathbf g_\boldnu(\mathbf R_t;\mathbfcal L_t,\mathbf s_{t-1})$ via \eqref{eq:s_gumbel};
$z_{t}(\boldnu,\boldpsi)$ via \eqref{zt}
\label{line:s_sample}
\;
}
Sample estimate of $\widehat{\rm ELBO}$ as ${\widehat{\rm ELBO}_{m}^{(n)}(\boldnu,\boldpsi)=\sum_{t=1}^{T}  z_{t}(\boldnu,\boldpsi) }$ (see \eqref{1hotELBO}) \label{line:ELBO} \;
{\bf Backward propagation}: sample gradient $\nabla_{\boldnu,\boldpsi} \widehat{\rm ELBO}_{m}^{(n)}(\boldnu,\boldpsi)$ via softmax approx.\label{line:grade} \; 
}
}
Batch gradient
$\nabla{=}\frac{1}{N\cdot N_{\rm trg } }\sum_{n,m} \nabla_{\boldnu,\boldpsi} \widehat{\rm ELBO}_{m}^{(n)}(\boldnu,\boldpsi)$  \label{line:avg_grad}\;
SGA step $(\boldnu,\boldpsi) \gets (\boldnu,\boldpsi)+\gamma\cdot\nabla$ \label{line:app_grad}\;
{Use $p_{\boldpsi}$ to compute prior belief updates for the next epoch (see \eqref{eq:prior_up});}
 ${\rm epoch} \gets{\rm epoch} +1$\;
 repeat from step 1 until convergence\;
 \Return{ $q_{\boldnu}, p_{\boldpsi}$ }
\end{algorithm2e}
\footnotetext{We do not express the dependence on $n,m$ for notational convenience, whenever unambiguous}
{Thanks to this reparameterization, 
by reorganizing the vectors $\mathbf g_\boldnu(\mathbf R_t;\mathbfcal L_t,\mathbf e_j)$ into the matrix
$\mathbf G_\boldnu(\mathbf R_t;\mathbfcal L_t)$ with columns
$[\mathbf G_\boldnu(\mathbf R_t;\mathbfcal L_t)]_{:,j}{=}
\mathbf g_\boldnu(\mathbf R_t;\mathbfcal L_t,\mathbf e_j)$,
we can then express the state sequence as an explicit function of $\boldnu$,
\begin{align}
\label{approxSt}
&\mathbf s_0=\mathbf g(\mathbf R_0;\mathbfcal L_0),\ 
\mathbf s_t{=}
\mathbf G_\boldnu(\mathbf R_t{;}\mathbfcal L_t)\cdot\mathbf s_{t-1},\forall t\geq 1.
\end{align}
}
Yet,
{computing the stochastic gradient of $\overline {\rm ELBO}(\boldnu,\boldpsi)$ with respect to $\boldnu$ requires computing the gradient of $\mathbf s_t$ with respect to $\boldnu$ -- a non-differentiable function due to the}
$\argmax$ in \eqref{eq:s_gumbel}. 
Similar to \cite{Gumbel}, we  adopt a hybrid strategy, where for forward-propagation we generate $\mathbf s_t$ via \eqref{approxSt} and for gradient calculation via back-propagation, we
approximate \eqref{eq:s_gumbel} through the (differentiable) softmax function, one-hot encoded as $\tilde{\mathbf g}_\boldnu(\mathbf R_t;\mathbfcal L_t,\mathbf s_{t-1})$ with components,
 $\forall i \in \mathcal S$,
\begin{align}
\label{eq:s_onehot_approx}
\!\![\tilde{\mathbf g}_\boldnu(\mathbf R_t;\mathbf e_j,\mathbfcal L_t)]_i{=}
\frac{\exp\Big\{\frac{\mathbf R_{t,i}{+}\ln q_{\boldnu}(i|j,\mathbfcal L_t)}{\tau}\Big\}}{\sum_{\ell{=}1}^{|\mathcal S|}\!\exp\!\Big\{\!\frac{\mathbf R_{t,\ell}{+}\ln q_{\boldnu}(\ell|j,\mathbfcal L_t)}{\tau}\!\Big\}},
 \end{align} 
where $\tau{>}0$ is a temperature parameter controlling the smoothness of the approximation: as $\tau{\to}0$,  $\tilde{\mathbf g}_\boldnu$ approaches the exact one-hot encoded function $g_\boldnu$ in~\eqref{eq:s_gumbel}.  
{
By replacing $\mathbf G_\boldnu$ with
$\tilde{\mathbf G}_\boldnu{=}
\sum_{j\in\mathcal S}
\tilde{\mathbf g}_\boldnu(\mathbf R_t;\mathbf e_j,\mathbfcal L_t)
\mathbf e_j^\top
$ in \eqref{approxSt},
we have thus defined a
differentiable approximation of $\mathbf s_t$, computed recursively as
$\nabla_{\boldnu}\mathbf s_t{=}
[\nabla_{\boldnu}\tilde{\mathbf G}_\boldnu(\mathbf R_t{;}\mathbfcal L_t)]\mathbf s_{t-1}
+\mathbf G_\boldnu(\mathbf R_t{;}\mathbfcal L_t)[\nabla_{\boldnu}\mathbf s_{t-1}]
$.
}
\begin{table*}[t]
\footnotesize
\begin{center}
\begin{tabular}{|l|c|l| |l|c|l| |l|c|l|}
\hline
  \# of UE beams & \!\!\!$|\mathcal F|$\!\!\! & \!\!\!$16$\!\!\!
  &
  BS height & \!\!\!$h_{\rm BS}$\!\!\! & \!\!\!$10$[m]\!\!\!
  &
  Misalignment to alignment gain ratio:\!\!\! & \!\!\!$\rho$\!\!\! & 
  \\
  \# of BS beams & \!\!\!$|\mathcal C|$\!\!\! & \!\!\!$32$\!\!\!
  &
  BS to road center distance & \!\!\!$D$\!\!\! & \!\!\!$22$[m]\!\!\!
  &
  \ \ \ \ Straight highway, LOS &  & \!\!\!$-10.2$dB\!\!\!
  \\
  \# of UE antennas & \!\!\!$M_{\rm rx}$\!\!\! & \!\!\!$(8{\times}4)$\!\!\!
  & 
  Lane separation & \!\!\!\!$\Delta_{\rm lane}$\!\!\!\! & \!\!\!$3.7$[m]\!\!\!
  &
  \ \ \ \ T-shaped urban, LOS+NLOS &  & \!\!\!$-8.2$dB\!\!\!\textbf{}
  \\
  \# of BS antennas & \!\!\!$M_{\rm tx}$\!\!\! & \!\!\!$(16{\times}8)$\!\!\!
  &
  UE average speed & \!\!\!$\mu_v$\!\!\! & \!\!\!30[m/s]\!\!\!
  &
  Max power ratio:  &  -  &  
  \\
  Bandwidth & \!\!\!$W_{\rm tot}$\!\!\! & \!\!\!$100$[MHz]\!\!\!
  &
  UE speed standard deviation & \!\!\!$\sigma_v$\!\!\! & \!\!\!10[m/s]\!\!\!
  &
  \ \ \ \ Straight highway, LOS  &    &  \!\!\!$3.8$dB\!\!\!
  \\
  Slot duration & \!\!\!$T_{\rm s}$\!\!\! & \!\!\!$400$[$\mu$s]\!\!\!
  &
  UE mobility memory parameter & \!\!\!$\gamma_v$\!\!\! & \!\!\!0.2\!\!\!
  &
  \ \ \ \ T-shaped urban, LOS+NLOS  &    &  \!\!\!$4.5$dB\!\!\!
  \\
  Frame duration [time]\!\!\!& \!\!\!$T_{\rm fr}$\!\!\! & \!\!\!20[ms]\!\!\!
  &
  UE lane change prob. (per frame)\!\!\! & \!\!\!$q$\!\!\! & \!\!\!0.01\!\!\!
  &
  Batch size (episodes per train. epoch)\!\!\!& \!\!\!$N$\!\!\! & \!\!\!5\!\!\!  
  \\
  Frame duration [slots]\!\!\! & \!\!\!$K$\!\!\! & \!\!\!$50$\!\!\!
  &
  BS height & \!\!\!$h_{\rm BS}$\!\!\! & \!\!\!$10$[m]\!\!\!
  &
  \# of state trajectories per episode  & \!\!\!$N_{\rm trg}$\!\!\! & \!\!\!$6$\!\!\!
  \\
  Carrier frequency & \!\!\!$f_c$\!\!\! & \!\!\!$30$[GHz]\!\!\!
  &
       &  & 
  &
  \# of sorted beliefs & \!\!\!$|\tilde{\mathcal B}^{\rm sort}|$\!\!\! & \!\!\!$2000$\!\!\!
  \\
  \hline
\end{tabular}
\normalsize
\caption{Parameters, symbols, and numerical values used in the simulations.
\vspace{-5mm}
}
\label{table1}
\end{center}
\end{table*}

\label{page:vae_train}   
{The overall training of DR-VAE is shown in Algorithm \ref{alg:VAE_train}. 
Starting from an initialization of $(\boldnu,\boldpsi)$, it
solves {\bf [LEARNING]} via SGA
over multiple \emph{epochs}, each corresponding to a batch of $N$ episodes, and
 returns the encoder and transition model trained on the batches upon reaching convergence.
 In each training epoch,
 a batch of $N$ episodes is used at step 1, either selected randomly from an \emph{offline database} or sampled online
 following adaptation policy $\pi$ (possibly, each episode is executed using a different policy).
 For each episode:
 1) the corresponding log-likelihoods
 $\mathbfcal L_{0:T^{(n)}}^{(n)}$ are computed;
 2) $N_{\rm trg}$ independent trajectories of $\mathbf s_{0:T}$ are generated
by the encoder using the current posterior model $q_\boldnu$ via Gumbel-softmax reparameterization (forward propagation of steps 4-6), and a sample estimate of
$\widehat{\rm ELBO}$ is computed in step 7;
3) the gradient of the estimated $\widehat{\rm ELBO}$ is computed via backward propagation, using the softmax approximation  (step 8).
After doing these steps for all episodes,
the batch gradient is computed in step 9,
one SGA step is performed to compute the new parameters, with step-size $\gamma$ (step 10).
The procedure is repeated from step 1 until convergence.
The learned model is then used in \eqref{eq:prior_up} to generate prior beliefs for the POMDP adaptation, in a continuous process of learning and adaptation.}
\vspace{-4mm}
\section{Numerical Results}
\label{sec:numres}
In this section, we present numerical results illustrating the performance of the proposed DR-VAE based learning framework and the proposed PBVI-, MDP- and ER-MDP-based adaptation policies.
Unless otherwise stated, the simulation parameters are summarized in Table \ref{table1}.
\vspace{-4mm}
\subsection{Simulation setup}
\label{sec:setup}
{We consider two mobility and channel scenarios (see Fig.\ref{fig:scenarios}):}

\noindent{\textbf{\emph{Straight highway, LOS}:} this scenario models a straight highway with two lanes separated by $\Delta_{\rm lane}$. The UE changes lanes with probability $q$ (per frame). The UE position along the direction of the road, $\omega_{t}$, evolves according to a Gauss-Markov mobility process with speed $\upsilon_{t}$ \cite{8673556}, 
used also in \cite{GMmobility,8275645,TVT2020} to model vehicular dynamics.
\label{page:GMmotivate}
It is modeled as
\begin{align}
\label{GMmodel}
\begin{cases}
\upsilon_t =\gamma_v \upsilon_{t-1}+(1-\gamma_v)\mu_v + \sigma_v\sqrt{1-\gamma_v^2}\cdot \tilde{\upsilon}_{t-1},\\
\omega_{t}=\omega_{t-1}+T_{\rm fr}\cdot \upsilon_{t-1},
\end{cases}
\end{align}
where $(\mu_v,\sigma_v)$  is the average speed and its standard deviation;  $\gamma_v$ is the memory parameter and
 $\tilde{\upsilon}_{t-1}{\sim}\mathcal N(0,1)$ is i.i.d. over frames. 
 For this scenario, we consider a purely LOS channel:
the UE's position at each frame ($\omega_t$ and current lane position) determines the LOS AoA and AoD pair $(\theta_t^{(0)},\phi_t^{(0)})$, and the pathloss ${\rm PL}_t$; the channel is then generated as in \eqref{eq:channel}, without NLOS components ($N_P=0$).}

\noindent{\textbf{\emph{T-shaped urban, LOS+NLOS}:}
this scenario
\label{page:scenario2}models a more complicated geometry and channel conditions, whereby the BS provides coverage to a T-shaped urban road as shown in Fig.~\ref{fig:scenarios}. The UE's mobility along the road follows the same Gauss-Markov mobility dynamics with lane changes as in the previous scenario, with one distinction: at the T-junction, the UE continues straight with 50\% probability, or
turns right on the same lane; after making the turn, the UE's mobility follows the same Gauss-Markov mobility dynamics with lane changes, until it exits the coverage area. In this scenario, the channel has $N_P{=}2$ NLOS paths in addition to the LOS one, each with $\sigma_{\ell,t}^2{=}0.1/{\rm PL}_t, \ell{=}1,2 $, so that
 the NLOS components contain 20\% the energy of the LOS path, consistent with experimental observations in \cite{channel_model}.
The LOS AoA, AoD and pathloss are functions of the UE's position as in the previous scenario. The NLOS AoA/AoD pair $(\theta_t^{(\ell)},\phi_t^{(\ell)}), \ell{=}1,2$ are generated uniformly over the unit sphere, i.i.d. over frames, i.e., the azimuth component follows ${\rm Uniform}[0,2\pi]$ and the elevation component follows ${\rm Uniform}[0,\pi]$. Then, the channel is found using \eqref{eq:channel} with $N_P{=}2$ (1 LOS + 2 NLOS paths).}

\begin{figure}[t]
	\centering
	\includegraphics[trim = 10 10 10 0,clip,width=\columnwidth]{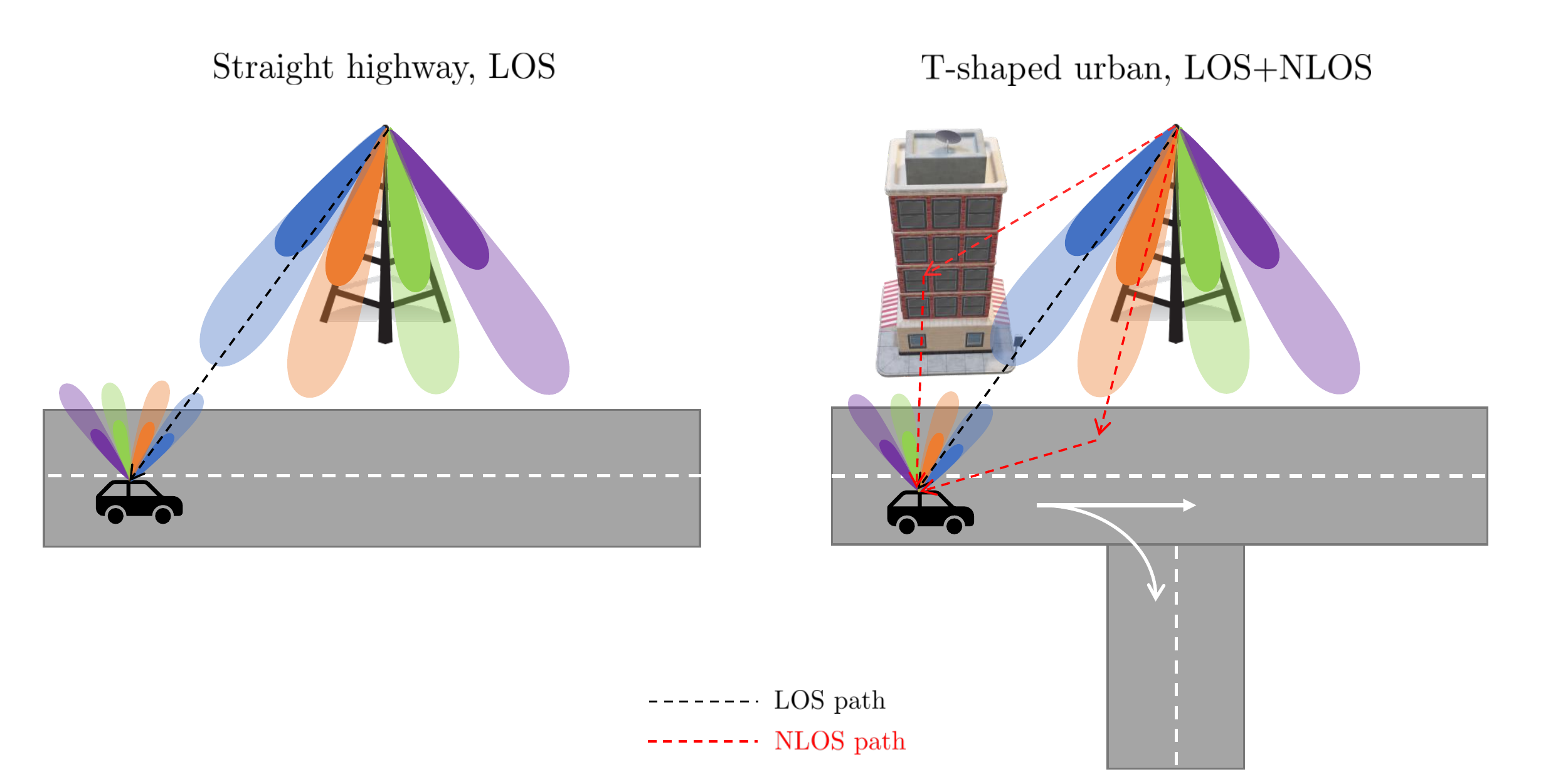}
	\caption{{Two UE's mobility and channel scenarios.}}\label{fig:scenarios}
\end{figure}

{For either scenarios,
the ground truth model of beam dynamics $p(\cdot|\cdot)$ is computed from several sequences of SBPIs $S_{0:T}$, by: 1) generating $10^4$ UE's trajectories;
2) for each trajectory, the SBPI sequence $S_{0:T}$ is computed via \eqref{eq:GLOS} and \eqref{Stdef}. 
 \label{page:GMMt2}
 Although the Gauss-Markov process exhibits the Markov property on the velocity-position pair $(\upsilon_t,\omega_t)$, the SBPI process $\{S_t,t{\geq}0\}$ is non-Markovian, since $S_t$ is a function of the
 lane and position $\omega_t$, and the speed $\upsilon_t$ introduces longer-term correlation that cannot be captured by the SBPI transition probability 
$p(s^{\prime}|s)$. As discussed next, our numerical evaluations reveal that, despite this mismatch, the Markov approximation on $\{S_t,t{\geq}0\}$ accurately represents the system performance.}

The BS and UE use uniform planar arrays with 3D analog beam steering, with codebooks
\begin{align*} 
\begin{cases}
    \mathcal C\!\!\! &= \left\{\mathbf d_{\rm tx}(\phi^{p}):p=1,\cdots,|\mathcal C| \right\}, 
    \\
    \mathcal F\!\!\! &=\left \{\mathbf d_{\rm rx}(\theta^{p}):p=1,\cdots,|\mathcal F| \right\},
    \end{cases}
\end{align*}
where $\theta^p$ and $\phi^p$ are the $p^{\rm th}$ AoA and AoD directions, respectively, {and $\mathbf d_{\rm tx}$, $\mathbf d_{\rm rx}$ are the associated array response vectors \cite{alkhateeb_hybrid}.}
{The set of AoD directions $\{\phi^p,\forall p\}$ are selected so that beam projections on the road corresponding to the half power bandwidth cover both lanes of the road.  The set of AoA directions $\{\theta^p,\forall p\}$ are selected uniformly spaced in $[0,\pi]^2$.}

The encoder $q_{\phi}$ of the DR-VAE is a recurrent neural network with one fully-connected hidden layer with 100 units, each with the relu ($\max(0,x)$) activation function. Similarly, the decoder $p_{\boldpsi}$ is a fully-connected feed-forward neural network with two hidden layers, each with 100 units and relu activation. 
For both encoder and decoder,
a softmax layer produces $|\mathcal S|$ output units. 
 {The DR-VAE is trained \emph{online} with Algorithm \ref{alg:VAE_train}:
 in each training epoch, the DR-VAE is updated using a batch of $N{=}5$ episodes (sampled using either the PBVI or (ER-)MDP policies) and for each episode $N_{\rm trg}{=}6$ trajectories of states are sampled to computed SGA updates;
 the model of beam dynamics is then updated via SGA and
used for the prior belief updates of the next epoch ($5$ episodes),  after which the DR-VAE is updated using the new episodes, and so on.}
 \begin{figure*}[ht]
   \label{figure:VAE_training}
\begin{subfigure}{.45\textwidth}
	\centering
	\includegraphics[trim = 5 30 15 10,clip,width=.9\columnwidth]{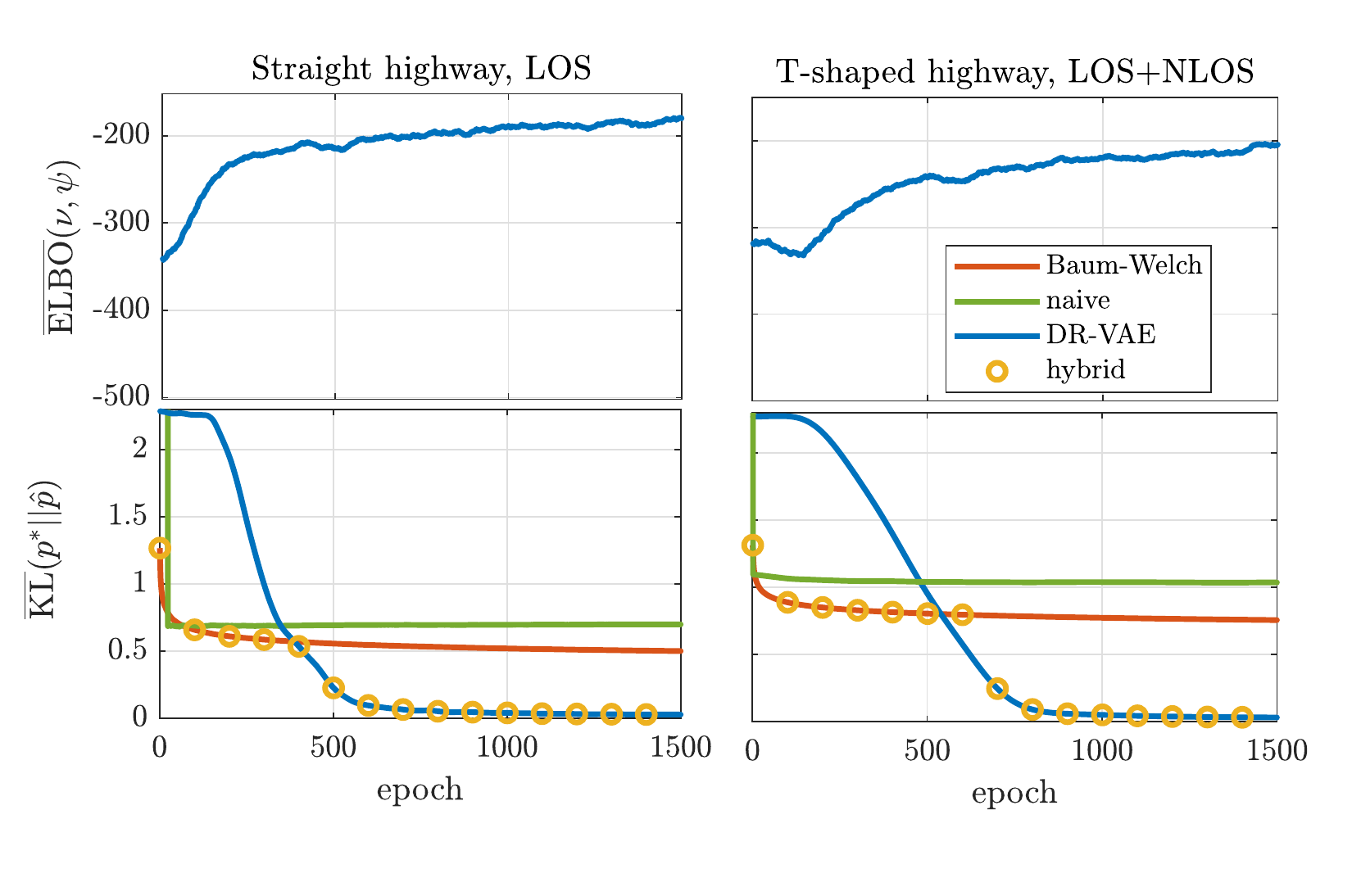}
	\caption{Training progress vs epochs {(1 epoch = 5 episodes)}.}
	\label{figure:VAE_training1}
\end{subfigure}\hspace{.5cm}
\begin{subfigure}{.45\textwidth}
	\centering
	\includegraphics[trim = 5 30 15 10,clip,width=.9\columnwidth]{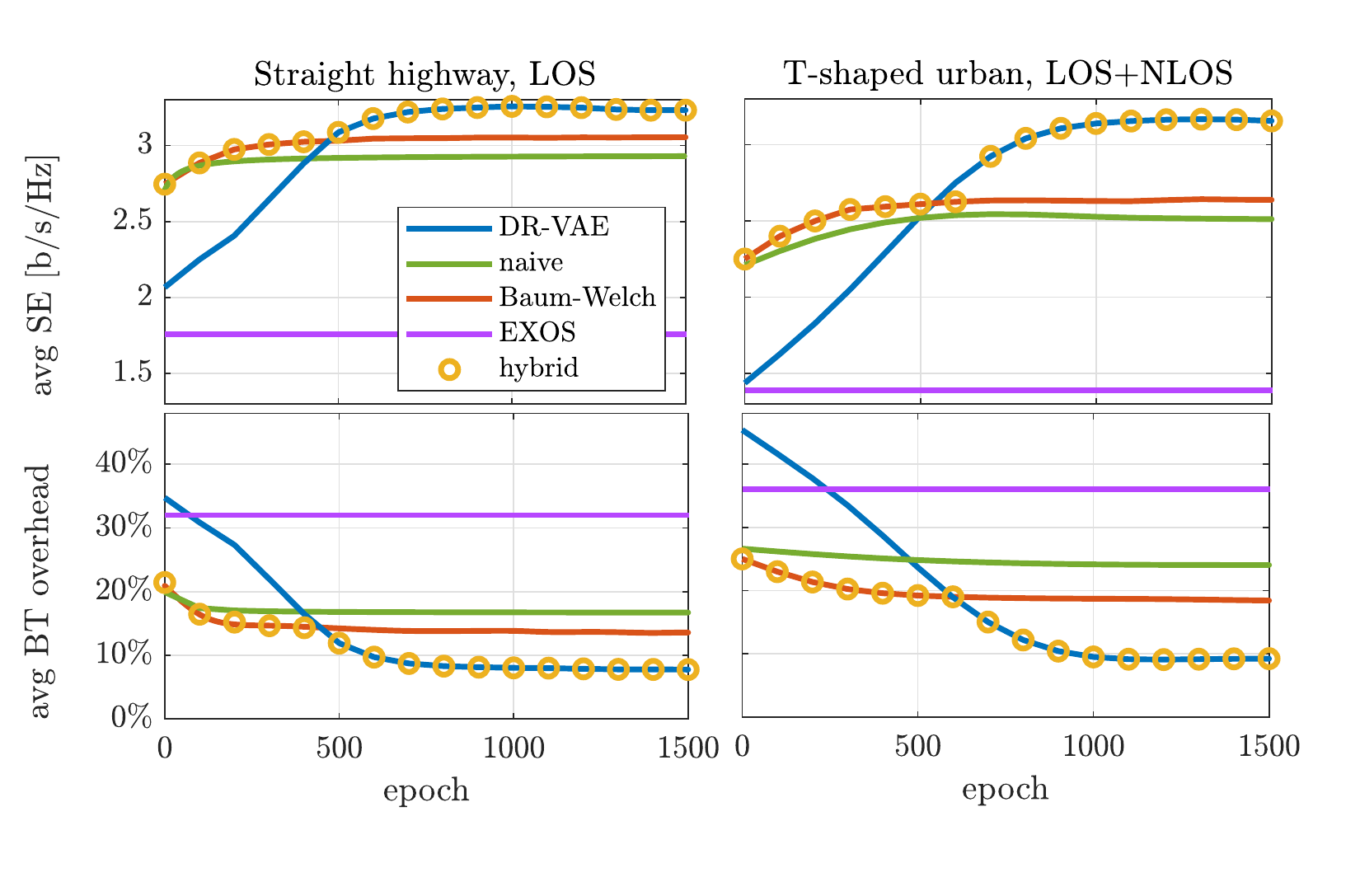}
	\caption{Average spectral efficiency and BT overhead vs epochs.}
	\label{figure:VAE_training2}
\end{subfigure}
\vspace{3mm}
\caption{Training progress of DR-VAE and other algorithms,
based on \emph{Gauss-Markov mobility and 3D analog beamforming}; ${\rm SNR}_{\rm BA} = 20$dB.\vspace{-4mm}}
\end{figure*}

We use the KL divergence  between the ground truth $p(\cdot|\cdot)$  and the learned model of beam dynamics $\hat p(\cdot|\cdot)$ to measure the accuracy of the learned model,
computed as
\begin{align}
 \overline {\rm KL}(p\|\hat p) \triangleq \frac{1}{|\mathcal S|} \sum_{s \in \mathcal S}\sum_{ s^{\prime} \in \mathcal S \cup \{ \bar s\}}  p(s^{\prime}|s) \ln \frac{p(s^{\prime}|s)}{\hat p(s^{\prime}|s)}.
 \end{align} 
 
 We evaluate the communications performance of a given adaptation policy $\pi$ through the
 average BT overhead (percentage of the frame duration used for BT) and
 average spectral efficiency $\bar{\rm SE}^\pi$ [bits/s/Hz]. By Little's law, it is expressed as
$\bar{\rm SE}^\pi = \frac{\bar B_{\rm tot}^{\pi}}{\bar D_{\rm tot} W_{\rm tot}},$
where $\bar B_{\rm tot}^{\pi}$ 
is the expected total number of bits successfully delivered to the UE under policy $\pi$ during the episode, and $\bar D_{\rm tot}$[s]
is the expected episode duration (independent of $\pi$, function only of UE's mobility).

{To evaluate these metrics, we use two types of Monte Carlo simulation approaches:}

\noindent{\emph{\bf Markov SBPI with binary SNR model},
in which we use the ground truth model of beam dynamics $p(\cdot|\cdot)$
to generate the SBPI sequence, and use the binary SNR model developed in \secref{secmodel} to assess the learning and communication metrics.}

\noindent{\emph{\bf 3D analog beamforming model with UE's mobility},
in which UEs' trajectories and the channel matrix are generated based on either scenario,
and the signal model \eqref{eq:signal_model} is used to evaluate 
the beam-training performance and outage conditions.
The purpose of this second approach is to validate the modeling abstractions used in our formulation (Markov SBPI dynamics, sectored antenna  with binary SNR model).}

In addition to the 
PBVI (Algorithm \ref{alg:PBVI_main}), MDP and ER-MDP policies (\secref{sec:MDP}),
we simulate two additional policies:
\label{page:EXOS}
\begin{itemize}[leftmargin=*]
\item Exhaustive search over SBPI (EXOS): it is an adaptation of exhaustive search proposed in \cite{exhaustive}; however,
unlike \cite{exhaustive}, which sweeps the entire AoA and AoD range (in our setting, a total of
$|\mathcal C|\cdot|\mathcal F|=512$ beam pairs, unfeasible within the frame duration), we restrict the exhaustive search to
the SBPIs only (i.e., those beam pairs that maximize the beamforming gain within the coverage area of the BS, see \eqref{sbpi}).
These represent a much smaller set of 15 beam pairs for the \emph{straight highway, LOS} scenario and 17 beam pairs for the \emph{T-shaped urban, LOS+NLOS} scenario. Note that no training is required for EXOS since it does not utilize the estimated beam dynamics.
\item Short-timescale single-shot (STSS) policy  \cite{stss}:
it is a state-of-the-art POMDP policy with a \emph{single-shot} beam training phase of \emph{fixed} duration.
Note that \cite{stss} assumes prior knowledge of the SBPI transition model; we use both the ground truth and the DR-VAE estimated models to evaluate STSS.
\end{itemize}
\subsection{Model training performance}
\label{sec:train_performance}
In \figref{figure:VAE_training1} and \figref{figure:VAE_training2}, we show the training progress of DR-VAE in terms of average ELBO, KL divergence,
spectral efficiency and BT overhead, based on the 
\emph{Gauss-Markov mobility and 3D analog beamforming}, for both scenarios. We compare different frameworks: the proposed DR-VAE (Algorithm \ref{alg:VAE_train}),
the naive approach (see \eqref{eq:naive}) and the Baum-Welch algorithm \cite{baum_welch}.
Observations are sampled following the PBVI-based policy (Algorithm \ref{alg:PBVI_main}).
We note that, as training progresses, the ELBO increases and $\overline {\rm KL}(p\|\hat p)$ decreases simultaneously, indicating an improvement in the accuracy of the learned transition model.
 At the same time, the average spectral efficiency  
 increases and the BT overhead decreases, indicating that the POMDP adaptation framework uses resources more efficiently by leveraging an accurate model of beam dynamics. 
 {Notably, thanks to  its error robustness, the PBVI policy yields better spectral efficiency
than EXOS, even in the early training epochs,
demonstrating that it provides a satisfactory baseline performance even with inaccurate models of beam dynamics;
as learning progresses, its performance improves even further.}
 \\\indent
 {Comparing the two scenarios, 
 we observe a slight performance degradation in
 the \emph{T-shaped urban, LOS+NLOS} scenario, attributed to
 NLOS channel components causing more frequent BT feedback errors (as also reflected by the increased  misalignment to alignment gain ratio $\rho$, see Table~\ref{table1}).}
 \begin{figure*}[h]
\begin{subfigure}{.45\textwidth}
	\centering
	\includegraphics[trim = 25 0 35 15,clip,width=.8\columnwidth]{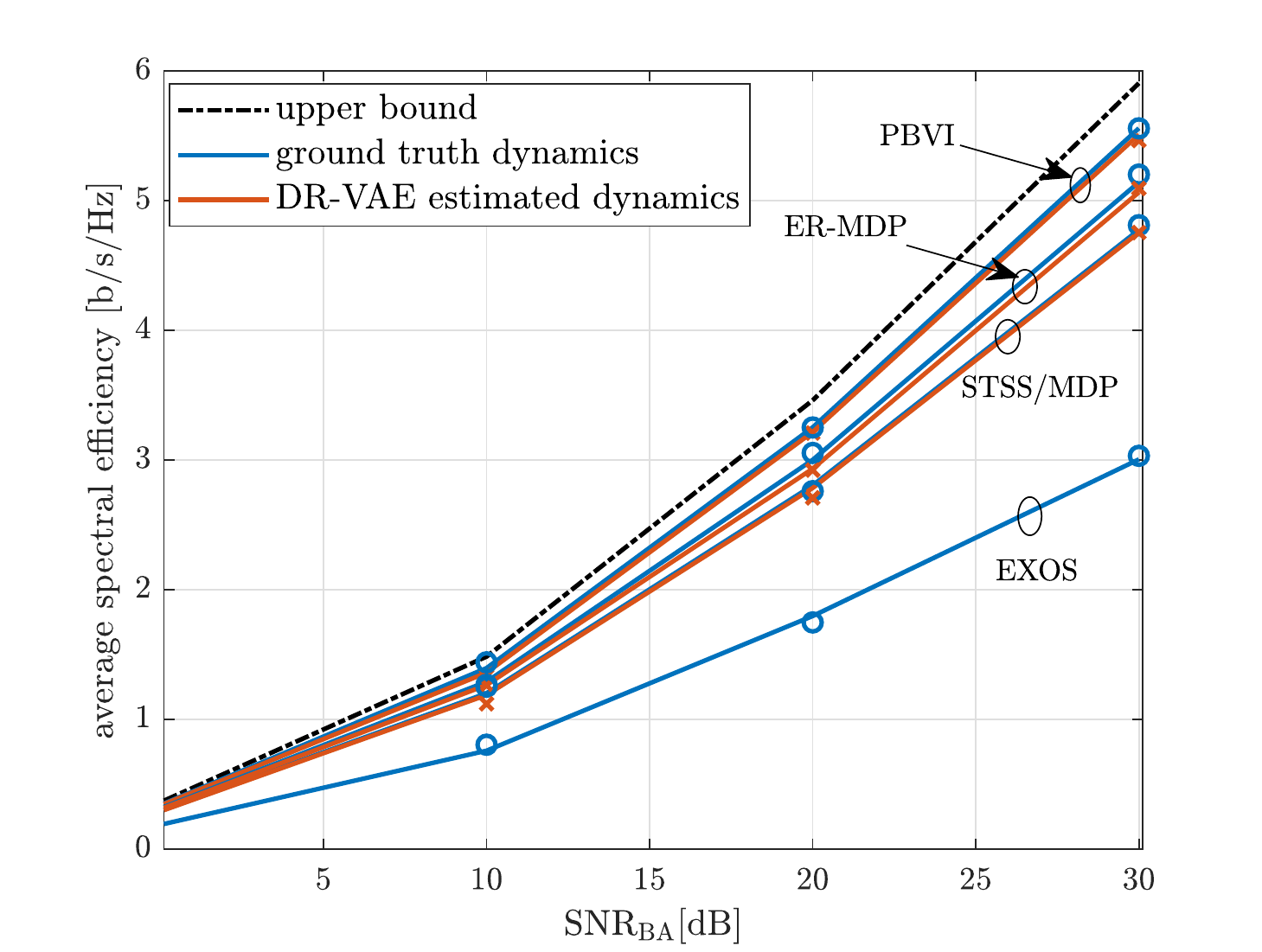}
	\caption{\emph{Straight highway, LOS} scenario}
	\label{figure:SE_vs_SNR}
\end{subfigure}\hspace{1cm}
\begin{subfigure}{.45\textwidth}
	\centering
	\includegraphics[trim = 25 8 30 7,clip,width=.8\columnwidth]{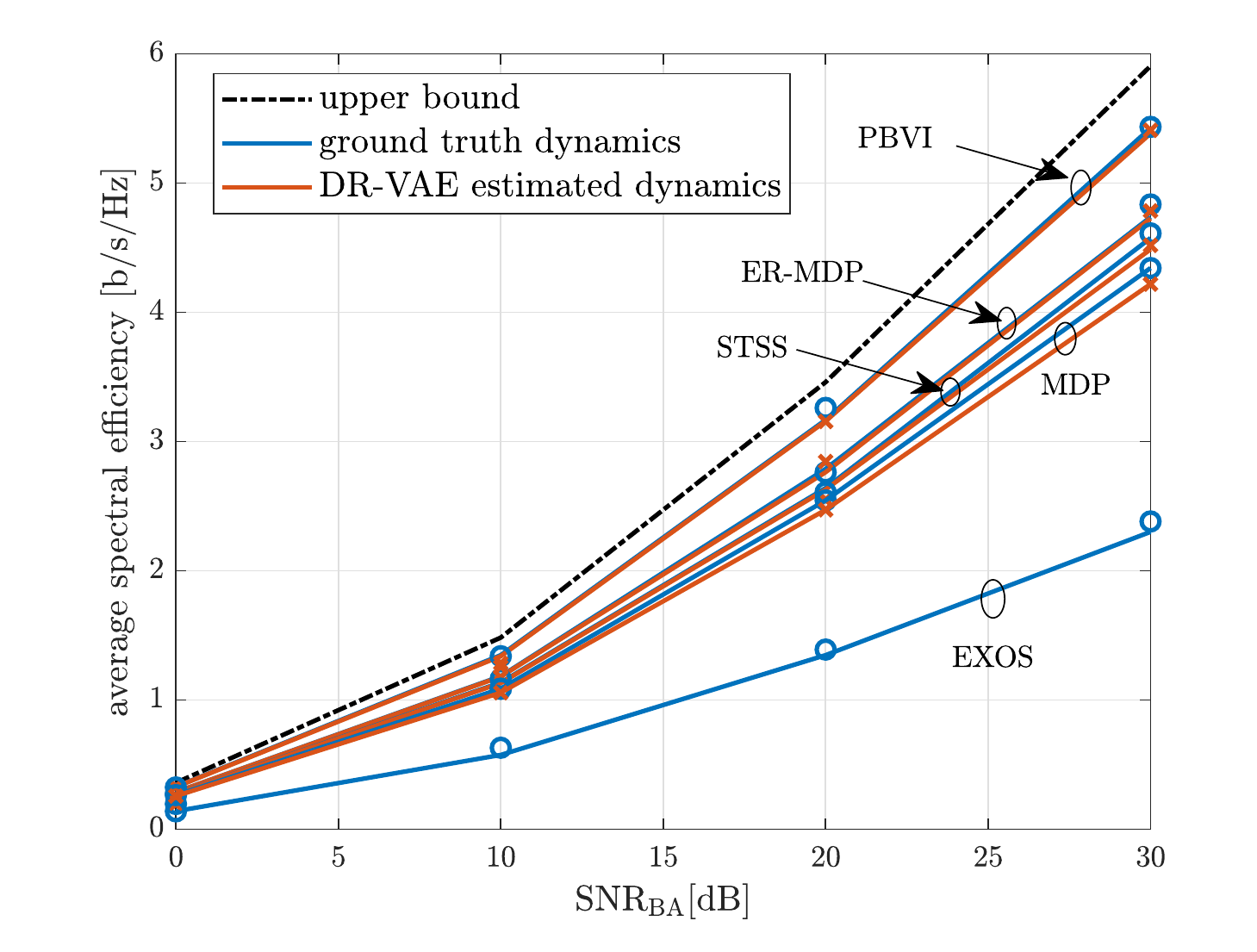}
	\caption{\emph{T-shaped urban, LOS+NLOS} scenario}
	\label{figure:SE_vs_SNR2}
\end{subfigure}
\vspace{3mm}
\caption{Average spectral efficiency versus SNR.
Monte Carlo simulation based on
\emph{Markov SBPI with binary SNR model} (solid lines)
 and
\emph{3D analog beamforming model with UE's mobility} (markers).\vspace{-2mm}}
\label{fig:fig}
\end{figure*}
 \begin{figure*}[h]
\begin{subfigure}{.45\textwidth}
        \centering
        \includegraphics[trim = 25 0 30 15,clip,width=.8\columnwidth]{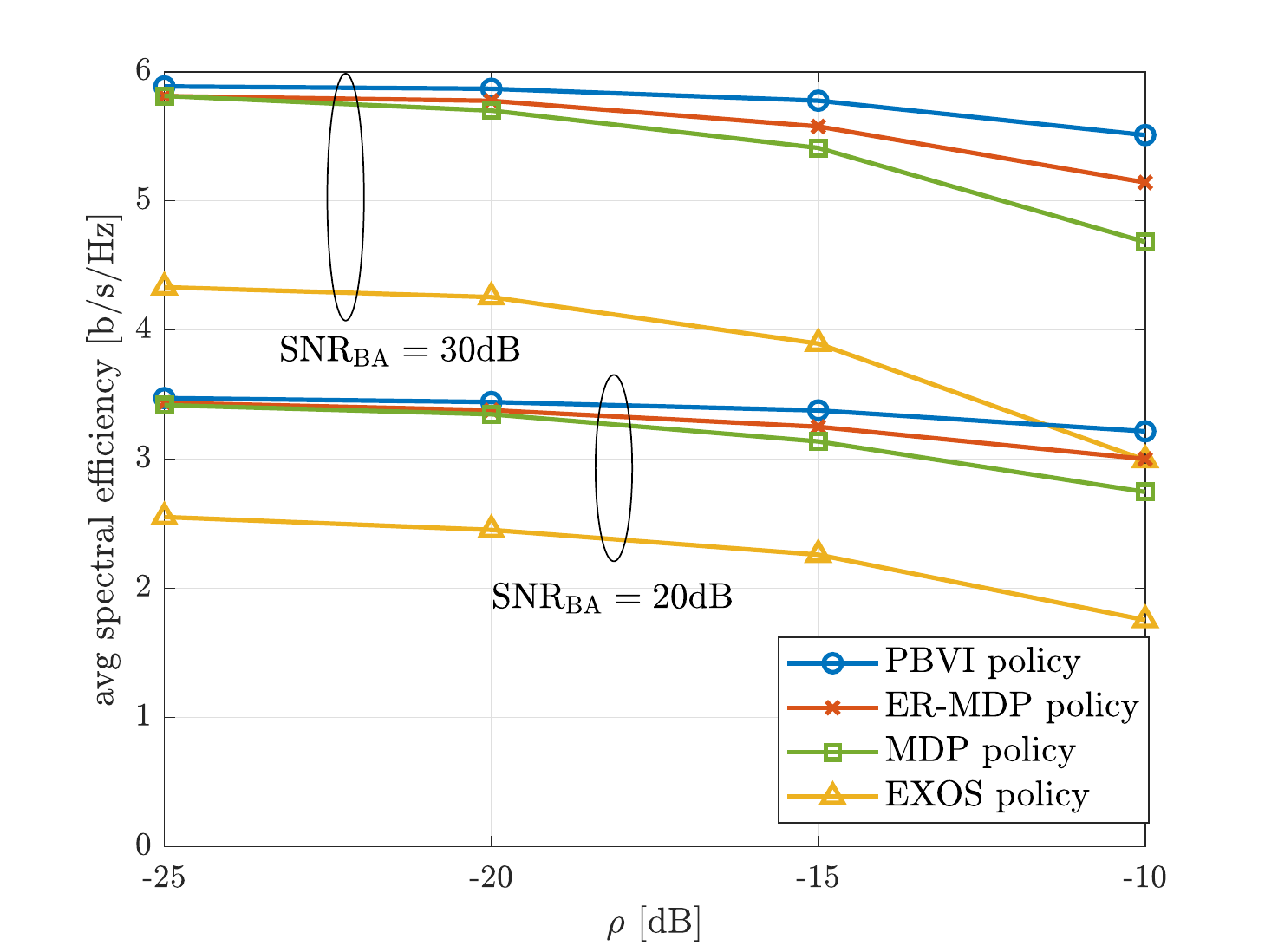}
	\caption{Average spectral efficiency versus $\rho$.}
	\label{figure:SE_vs_rho}
\end{subfigure}\hspace{1cm}
\begin{subfigure}{.45\textwidth}
        \centering
        \includegraphics[trim = 20 0 30 15,clip,width=.8\columnwidth]{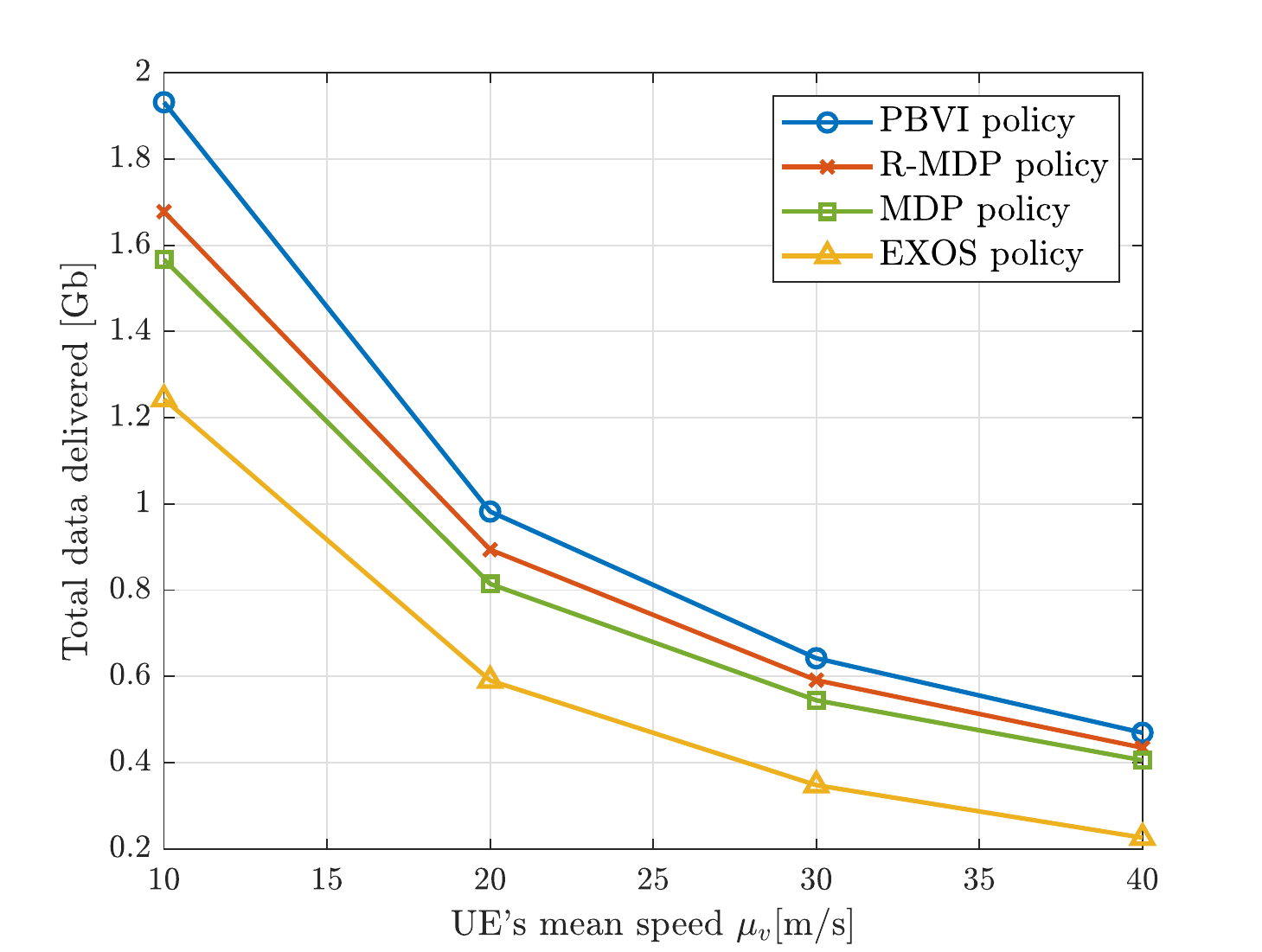}
	\caption{Total data delivered vs mean speed $\mu_v$; ${\rm SNR}_{\rm BA} = 20$dB, $\rho = -10$dB.}
	\label{figure:Thput_vs_rho}
\end{subfigure}
\vspace{2mm}
\caption{
Monte-Carlo simulations based on:
\emph{Markov SBPI with binary SNR model}; ground truth model of beam dynamics; \emph{straight highway, LOS} scenario.\vspace{-4mm}}
\label{fig:fig}
\end{figure*} 

  Comparing DR-VAE with other algorithms, we note that
  Baum-Welch and the naive approach converge faster initially, but get stuck at a suboptimal point as training progresses: upon reaching convergence,
 DR-VAE offers ${\sim}90\%$ (\emph{straight highway, LOS}) and ${\sim} 95\%$ (\emph{T-shaped urban, LOS+NLOS})  reduction in KL divergence compared to the other two methods. 
  This improved accuracy translates into a ${\sim} 10\%$ and ${\sim} 20\%$ 
  spectral efficiency gain 
  in the two scenarios, respectively, and ${\sim} 50\%$ reduction of BT overhead, as depicted in the figure.
  Motivated by the improved early convergence of Baum-Welch, we propose a \emph{hybrid approach}:
in the early learning stages, 
the POMDP adaptation framework uses the
prior beliefs generated based on the model learned by the Baum-Welch algorithm; when the ELBO metric converges, 
it switches to the model provided by the DR-VAE.
   Such a hybrid scheme is possible thanks to the decoupling of policy design and learning of beam dynamics via the dual timescale approach:
   we can thus train multiple models of beam dynamics simultaneously and use the prior belief generated based on any one model.
   As can be seen in the figure (markers), this hybrid approach achieves the best performance
   across all metrics.

\subsection{Spectral Efficiency}
\label{sec:SE}
{In Fig. \ref{figure:SE_vs_SNR} and \ref{figure:SE_vs_SNR2}, we present the spectral efficiency performance as a function of the target SNR for the two scenarios, using}
Monte-Carlo simulation over $10^5$ episodes.
In addition to the three proposed policies (PBVI, MDP and ER-MDP policies), the state-of-the-art STSS \cite{stss} and EXOS, we also evaluate a genie-aided upper bound, computed
using error-free BT feedback and knowledge of the ground truth model of beam dynamics.
 Note that this upper bound is not attainable in practice, due to feedback errors
 and inaccuracies in the learned model.
To assess the impact of the DR-VAE learning algorithm on performance,
all these schemes (except EXOS and the upper bound)
are evaluated using both the DR-VAE estimated model of beam dynamics for the belief updates, computed after  convergence of the learning algorithm (red curves), and also using the ground truth model of beam dynamics (blue curves).

We observe that, for all the schemes, the DR-VAE performs very close to the counterpart using the ground-truth model of beam dynamics, thus confirming the results of \figref{figure:VAE_training1} and \figref{figure:VAE_training2}.
The PBVI policy coupled with DR-VAE yields the best performance, close to the genie-aided upper bound. {For the \emph{straight highway, LOS} scenario, it outperforms the ER-MDP, STSS/MDP (shown together since they exhibit nearly identical performance), and EXOS by up to $8\%$, $16\%$ and $85\%$ in spectral efficiency, respectively. For the \emph{T-shaped urban, LOS+NLOS} scenario, PBVI (coupled with DR-VAE)
shows even bigger gains:
it outperforms ER-MDP, STSS, MDP, and EXOS by up to $14\%$, $20\%$, $28\%$ and $130\%$ in spectral efficiency, respectively.} 

 The gain of PBVI over the other policies is attributed to its enhanced robustness against feedback errors, incorporated via the BT feedback distribution, and its optimized and adaptive BT design.
 Although ER-MDP suffers from a slight performance degradation with respect to PBVI, it outperforms
 all other schemes: 
unlike the MDP-based policy, ER-MDP incorporates error-robustness by using POMDP belief updates, and unlike STSS that uses a fixed BT duration, ER-MDP
adjusts the BT overhead adaptively by leveraging the BT feedback.
This is a remarkable result given that ER-MDP is a low-complexity policy, compared to the POMDP-based STSS.

To assess the validity of the modeling abstractions, we also evaluate the performance based on
the \emph{Markov SBPI with binary SNR model} (solid lines)
 and the
\emph{3D analog beamforming model with UE's mobility} (markers). 
 It can be seen that the values under the two approaches match, thereby verifying the
  accuracy of the proposed modeling abstractions.
  \\\indent
In \figref{figure:SE_vs_rho}, we depict the spectral efficiency 
vs the
misalignment to alignment gain ratio $\rho$ for
the \emph{straight highway, LOS} scenario, under the
\emph{Markov SBPI with binary SNR model}.
We use the ground truth model for this evaluation, since our previous results demonstrated that DR-VAE learns the model accurately.
{Note that larger values of $\rho$ account for the effect of more severe NLOS multipath and sidelobes, causing more frequent feedback errors.}
As expected, the performance of the four policies degrades as $\rho$ increases.
Yet, notably, PBVI degrades the least thanks to its robustness to errors, whereas EXOS degrades the most. Moreover, for small to medium values of $\rho$, ER-MDP performs very close to PBVI, since feedback errors become less frequent.
At the same time, we found that the total optimization and execution time of ER-MDP is ${\sim}5$ times smaller than PBVI, so that ER-MDP offers a low-complexity alternative to PBVI for small values of $\rho$. \\\indent
In \figref{figure:Thput_vs_rho}, we depict the average total data delivered to the UE  successfully  as a function of the mean UE speed $\mu_v$.
We note a monotonically decreasing trend with the mean speed $\mu_v$, attributed to the shorter average episode duration as speed increases, and the exacerbated overhead of BT since SBPIs change more frequently.
The behavior is in line with what we observed previously:
 PBVI outperforms ER-MDP and MDP-based policies, which in turn outperform EXOS. 

 \vspace{-2mm}
\section{Conclusion and Future Work}
\label{sec:conc}
This paper proposed a dual timescale learning and adaptation framework,
in which the beam dynamics are learned to enable predictive beam-tracking, and then exploited to design adaptive beam-training policies.
In the short-timescale,
we developed a POMDP framework to design an approximately optimal policy.
In the long-timescale,
 we designed a deep recurrent variational autoencoder-based learning framework, which uses noisy observations collected under the policy to learn a transition model of beam dynamics. 
Via simulation,
we demonstrated the superior learning performance of the proposed learning framework over the Baum-Welch algorithm and 
a naive learning method, with spectral efficiency gains of $\sim10\%$ and reduced beam-training overhead by $\sim 50\%$.
Our performance evaluation demonstrated that the proposed policy, coupled with the learning framework, yields  near-optimal performance, with
${\sim}16\%$ spectral efficiency gains over a state-of-the-art POMDP policy. 
This work paves the way to future research on beam tracking design in multi-base station and multi-user settings,
by learning and leveraging the joint beam dynamics of multiple users, as well as the extension to hybrid beamforming architectures.

\appendices
\section*{Appendix A: Proof of Theorem \ref{strucPOMDP}}
\begin{proof}
We use induction to prove P1.
P1 holds for $k{=}K$ since $\mathcal Q_K{=}\{\mathbf 0\}{=}\mathcal P(\mathbf 0)$. 
 Assume $\mathcal Q_j$ satisfies P1 for $j{>}k$, for some $k{\in}\mathcal K$.
We show that it implies $\mathcal Q_k$ satisfies P1 as well. Let $\bfalpha{\in}\mathcal Q_k$, with $\mathcal Q_k$ given by \eqref{eq:Q_val_iter}. If $\bfalpha{=}{\rm SE}_{\rm BA}{\cdot}\left(1-\frac{k}{K}\right)\mathbf e_\ell$, then
$
\mathcal P(\bfalpha)=
\cup_{i=1}^{|\mathcal S|}\{{\rm SE}_{\rm BA}{\cdot}(1-k/K)\mathbf e_i\}
\subseteq \mathcal Q_k.
$
Now, consider ${\bfalpha}{=}\sum_{y \in \mathcal Y} \mathbb P_Y(y|\cdot,{\mathcal S}_{\rm BT}){\odot}{\bfalpha}^{(y)}$, 
for some $[\bfalpha^{(y)}]_{y \in \mathcal Y} \in \mathcal Q_{k+n+1}^{n+1}$ and ${\mathcal S}_{\rm BT}\in\mathcal A_{{\rm BT},k}$ with $|{\mathcal S}_{\rm BT}|=n$, so that
 $\bfalpha\in\mathcal Q_k$.
Let ${\rm prm}(\cdot)$ be a generic permutation operator, and $\tilde{\bfalpha}{\triangleq}{\rm prm}(\bfalpha){\in}\mathcal P(\bfalpha)$.
We will show that $\tilde{\bfalpha}{\in}\mathcal Q_k$, hence
 $\mathcal P(\bfalpha)\subseteq\mathcal Q_k$, which proves the induction step and P1.
 
 To prove $\tilde{\bfalpha}{\in}\mathcal Q_k$,
 we need to show that there is a 
 BT action over the BPI set $\tilde{\mathcal S}_{\rm BT}{\in}\mathcal A_{{\rm BT},k}$ and vectors $[\tilde{\bfalpha}^{(\tilde y)}]_{\tilde y \in \tilde{\mathcal Y}}{\in}\mathcal Q_{k+n+1}^{n+1}$ such that $\tilde{\bfalpha}{=}\sum_{\tilde y \in \tilde{\mathcal Y}}  \mathbb P_Y(\tilde y|\cdot,\tilde{\mathcal S}_{\rm BT}){\odot} \tilde{\bfalpha}^{(\tilde y)}$,
where $\tilde{\mathcal Y}{\equiv}\tilde{\mathcal S}_{\rm BT}{\cup}\{0\}$ is the corresponding observation set.
Let $\chi{:}\mathcal S{\mapsto}\mathcal S$ be the function that maps a state $s$ to its permuted state $s'{=}\chi(s)$ through the operator ${\rm prm}(\cdot)$, and let $\Lambda(\cdot)$ be the inverse mapping ($s{=}\Lambda(s')$), so that
for $\tilde{\mathbf a}{=}{\rm prm}(\mathbf a)$ we have 
 $\tilde{\mathbf a}(\chi(s)){=}\mathbf a(s)$  and $\tilde{\mathbf a}(s){=}\mathbf a(\Lambda(s))$.
We define the new BT set $\tilde{\mathcal S}_{\rm BT}{=}\{\chi(s){:}s{\in}{\mathcal S}_{\rm BT}\}$ and new vectors
$\tilde{\bfalpha}^{(0)}{=}{\rm prm}({\bfalpha}^{(0)})$,
$\tilde{\bfalpha}^{(\chi(y))}{=}{\rm prm}({\bfalpha}^{(y)}),\forall y{\in}{\mathcal S}_{\rm BT}$. Clearly, 
$\tilde{\bfalpha}^{(\tilde y)}{\in}\mathcal Q_{k+n+1},\forall \tilde y{\in}\tilde{\mathcal Y}$, from the induction hypothesis, and $\tilde{\mathcal S}_{\rm BT}{\in}\mathcal A_{{\rm BT},k}$,
hence $\sum_{\tilde y \in \tilde{\mathcal Y}} \mathbb P_Y(\tilde y|\cdot,\tilde{\mathcal S}_{\rm BT}){\odot}\tilde{\bfalpha}^{(\tilde y)}{\in}\mathcal Q_k$.
It remains to prove that 
$\tilde{\bfalpha}{=}\sum_{\tilde y \in \tilde{\mathcal Y}} \mathbb P_Y(\tilde y|\cdot,\tilde{\mathcal S}_{\rm BT}){\odot} \tilde{\bfalpha}^{(\tilde y)}$, i.e.,
$\sum_{\tilde y \in \tilde{\mathcal Y}} \mathbb P_Y(\tilde y|s,\tilde{\mathcal S}_{\rm BT}) \tilde{\bfalpha}^{(\tilde y)}(s){=}\bfalpha(\Lambda(s)),\forall s$, since $\tilde{\bfalpha}{=}{\rm prm}(\bfalpha)$. 
In fact, $\forall s\in\mathcal S$, and letting $\Lambda(0)=\chi(0)=0$,
\begin{align*}
&\sum_{\tilde y \in \tilde{\mathcal Y}} \mathbb P_Y(\tilde y|s,\tilde{\mathcal S}_{\rm BT})\tilde{\bfalpha}^{(\tilde y)}(s)
\stackrel{(a)}{=}
\sum_{\tilde y \in \tilde{\mathcal Y}} \mathbb P_Y(\tilde y|s,\tilde{\mathcal S}_{\rm BT}){\bfalpha}^{(\Lambda(\tilde y))}(\Lambda(s))
\\&
\stackrel{(b)}{=}
\sum_{y \in {\mathcal S}_{\rm BT}\cup\{0\}}
\underbrace{\mathbb P_Y(\chi(y)|s,\tilde{\mathcal S}_{\rm BT})}_{\stackrel{(c)}{=}\mathbb P_Y(y|\Lambda(s),{\mathcal S}_{\rm BT})}{\bfalpha}^{(y)}(\Lambda(s))
\stackrel{(d)}{=}\bfalpha(\Lambda(s)),
\end{align*}
where (a) follows from  the definition of $\tilde{\bfalpha}^{(\tilde y)}$;
(b) from $\tilde{\mathcal S}_{\rm BT}=\{\chi(y):y{\in}{\mathcal S}_{\rm BT}\}$ and $\Lambda(\chi(y)){=}y$;
(c) from
the symmetry in the observation model;
 (d) by inspection. P1 is thus proved.

P1 implies that $\mathcal Q_{k}{\equiv}\cup_{\bfalpha\in\mathcal Q_{k}^{\rm sort}}\mathcal P(\bfalpha)$.
P2 then follows since
 $V_k^*(\bfbeta){=}
 \max\limits_{\bfalpha \in\mathcal Q_{k}^{\rm sort}}
 \max\limits_{\tilde{\bfalpha}\in\mathcal P(\bfalpha)}\inprod{\bfbeta }{\tilde{\bfalpha}}
=
\max\limits_{\bfalpha \in\mathcal Q_{k}^{\rm sort}}
\inprod{\mathrm{sort}(\bfbeta)}{\bfalpha}$
and
$\mathrm{sort}(\bfbeta'){=}\mathrm{sort}(\bfbeta),\forall\bfbeta'{\in}\mathcal P(\bfbeta)$.
 \end{proof}
 \vspace{-6mm}
\section*{Appendix B: Proof of Theorem \ref{thm:MDP}}
\begin{proof}
We prove the theorem by induction. 
P1-P4 hold trivially for $k{=}K$ since $V_{K}^*(\mathcal U){=}0$.
Assume P1-P4 hold for $j{>}k$, for some $k{<}K$. We will prove they hold for $k$ as well.

\noindent{\bf P1}: when $\mathcal U{=}\{s\}$, using the induction hypothesis P1 we obtain
$V_{k}^{(\mathrm{BT})}(\{s\},{\mathcal S}_{\rm BT}){=}{\rm SE}_{\rm BA}{\cdot}[1{-}(k{+}1)/K]<V_{k}^{(\mathrm{DC})}(\{s\}){=}{\rm SE}_{\rm BA}{\cdot}(1{-}k/K)$.
  Hence, the optimal action is DC with value
  $V_k^*(\{s\})={\rm SE}_{\rm BA}{\cdot}(1-k/K)$.

\noindent{\bf P2}:   
$V_k^*(\mathcal U){\geq}V_{k+1}^*(\mathcal U)$ follows from
the value iteration algorithm, since 
$1{-}k/K{>}1{-}(k{+}1)/K$,
$\mathcal A_{{\rm BT},k}{\supseteq}\mathcal A_{{\rm BT},k+1}$,
and $V_{k+|{\mathcal S}_{\rm BT}|+1}^*(\mathcal U){\geq}V_{k{+}1{+}|{\mathcal S}_{\rm BT}|{+}1}^*(\mathcal U)$ (induction hypothesis).

\noindent{\bf P3}: consider ${\mathcal S}_{{\rm BT}}^{(0)}$ such that $|{\mathcal S}_{{\rm BT}}^{(0)}{\setminus}\mathcal U|{\geq}1$,
and a new BT set ${\mathcal S}_{{\rm BT}}^{(1)}{=}{\mathcal S}_{{\rm BT}}^{(0)}{\cap}\mathcal U$. Since 
$\mathcal U{\cap}{\mathcal S}_{{\rm BT}}^{(0)}{=}\mathcal U{\cap}{\mathcal S}_{{\rm BT}}^{(1)}$, and
$\mathcal U{\setminus}{\mathcal S}_{{\rm BT}}^{(0)}{=}\mathcal U{\setminus}{\mathcal S}_{{\rm BT}}^{(1)}$,
P2 and $|{\mathcal S}_{{\rm BT}}^{(1)}|{<}|{\mathcal S}_{{\rm BT}}^{(0)}|$
imply suboptimality of ${\mathcal S}_{{\rm BT}}^{(0)}$.
  
\noindent  {\bf P4}: 
  consider $\mathcal U{\equiv}\{s_1,s_2,{\cdots},s_{|\mathcal U|}\}$ with $|\mathcal U|{\geq}2$,
  and a BT set ${\mathcal S}_{{\rm BT}}^{(1)}{=} \{s_1,s_2,{\cdots},s_n\}{\subset}\mathcal U$.
  Without loss of generality, assume that $\bfbeta(s_1){\geq}\bfbeta(s_2){\geq}\dots\bfbeta(s_n)$, 
  $\bfbeta(s_{n+1}){>}\bfbeta(s_n)$
  and $\bfbeta(s_{n+1}){\geq}\bfbeta(s_{n+2}){\geq}\dots\bfbeta(s_{|\mathcal U|})$; 
  in other words, ${\mathcal S}_{{\rm BT}}^{(1)}$ does not scan the $n$ most likely BPIs, since 
  $\bfbeta(s_{n+1}){>}\bfbeta(s_n)$.
  Let a new set ${\mathcal S}_{{\rm BT}}^{(0)}{=} \{s_1,s_2,\cdots,s_{n-1}\}{\cup}\{s_{n+1}\}{\subset}\mathcal U$,
 which scans the more likely $s_{n+1}$ rather than $s_n$.
  We want to show that 
  $V_{k}^{(\mathrm{BT})}(\mathcal U,{\mathcal S}_{{\rm BT}}^{(0)}){\geq}V_{k}^{(\mathrm{BT})}(\mathcal U,{\mathcal S}_{{\rm BT}}^{(1)})$, i.e., ${\mathcal S}_{{\rm BT}}^{(1)}$ is suboptimal.
  Using $\mathcal U{\cap}{\mathcal S}_{{\rm BT}}^{(i)}{\equiv}{\mathcal S}_{{\rm BT}}^{(i)}$,
$\mathcal U{\setminus}{\mathcal S}_{{\rm BT}}^{(i)}{\equiv}\{s_{n+i}\}\cup\{s_{n+2},\dots,s_{|\mathcal U|}\}{\triangleq}\mathcal U^{(i)}$,
and the induction hypothesis P1, we obtain
  \begin{align}
  \label{dfgh}
  \nonumber
V_{k}^{(\mathrm{BT})}(\mathcal U,{\mathcal S}_{{\rm BT}}^{(i)})=&
\frac{\sum_{s\in{\mathcal S}_{{\rm BT}}^{(i)}}\bfbeta(s)}{\sum_{s\in\mathcal U}\bfbeta(s)}{\rm SE}_{\rm BA}{\cdot}\Big(1{-}\frac{k{+}n{+}1}{K}\Big)
\\&
+\frac{\sum_{s\in\mathcal U^{(i)}}\bfbeta(s)}{\sum_{s\in\mathcal U}\bfbeta(s)}V_{k+n+1}^*(\mathcal U^{(i)}).
  \end{align}
    Now, consider $V_{k+n+1}^*(\mathcal U^{(1)})$. If it is optimized by DC, then  
\begin{align*}
&  V_{k{+}n{+}1}^*(\mathcal U^{(1)}){=}
V_{k{+}n{+}1}^{(\mathrm{DC})}(\mathcal U^{(1)}){=}
  \frac{\bfbeta(s_{n{+}1})}{\!\!\!\sum\limits_{s\in\mathcal U^{(1)}}\!\!\!\bfbeta(s)}{\rm SE}_{\rm BA}\Big(\!1{-}\frac{k{+}n{+}1}{K}\!\Big)\!,\\&
  V_{k+n+1}^*(\mathcal U^{(0)})
  {\geq}
  V_{k{+}n{+}1}^{(\mathrm{DC})}(\mathcal U^{(0)}){\geq}
  \frac{\bfbeta(s_{n})}{\!\!\!\!\sum\limits_{s\in\mathcal U^{(0)}}\bfbeta(s)}{\rm SE}_{\rm BA}\Big(\!1{-}\frac{k{+}n{+}1}{K}\!\Big)\!,
\end{align*}
where we used
 $\max\limits_{s\in\mathcal U^{(1)}}\bfbeta(s){=}\bfbeta(s_{n+1})$ and
$\max\limits_{s\in\mathcal U^{(0)}}\bfbeta(s){\geq}\bfbeta(s_{n})$;
the inequality holds since choosing DC may be suboptimal for  $V_{k+n+1}^*(\mathcal U^{(0)})$.
Hence, $V_{k}^{(\mathrm{BT})}(\mathcal U,{\mathcal S}_{{\rm BT}}^{(0)})
-V_{k}^{(\mathrm{BT})}(\mathcal U,{\mathcal S}_{{\rm BT}}^{(1)})\geq 0$.

 Now, consider the case when $V_{k+n+1}^*(\mathcal U^{(1)})$ is optimized by the BT action over the BPI set $\mathcal S_{{\rm BT}}^{(1)}{\equiv}\{s_{n+1},\dots,s_{n+m}\}$ (from the induction hypothesis P4, choosing the $m$ most likely BPIs is optimal), 
and let $\mathcal S_{{\rm BT}}^{(0)}{\equiv}\{s_n\}{\cup}\{s_{n+2},\cdots,s_{n+m}\}$.
Since
$\mathcal U^{(i)}{\cap}{\mathcal S}_{{\rm BT}}^{(i)}{\equiv}{\mathcal S}_{{\rm BT}}^{(i)}$,
$\mathcal U^{(1)}{\setminus}{\mathcal S}_{{\rm BT}}^{(1)}{\equiv}\mathcal U^{(0)}{\setminus}{\mathcal S}_{{\rm BT}}^{(0)}{\triangleq}\mathcal U^{(2)}$,
it follows
\begin{align*}
&  V_{k+n+1}^*(\mathcal U^{(i)})
  {\geq}
V_{k{+}n{+}1}^{(\mathrm{BT})}(\mathcal U^{(i)},{\mathcal S}_{{\rm BT}}^{(i)})
{=}
\frac{\sum_{s\in\mathcal U^{(2)}}\bfbeta(s)}{\sum_{s\in\mathcal U^{(i)}}\bfbeta(s)}
\\&
{\times}V_{k+n+m+2}^*(\mathcal U^{(2)})
{+}\frac{\sum_{s\in{\mathcal S}_{{\rm BT}}^{(i)}}\bfbeta(s)}{\sum_{s\in\mathcal U^{(i)}}\bfbeta(s)}
{\rm SE}_{\rm BA}{\cdot}\left(\!1{-}\frac{k{+}n{+}m{+}2}{K}\!\right)\!,
\end{align*}
with equality for $i{=}1$ (since $V_{k+n+1}^*(\mathcal U^{(1)})$ is optimized by BT over the set $\mathcal S_{{\rm BT}}^{(1)}$),
and the inequality for $i{=}0$ holds since choosing BT with set 
${\mathcal S}_{{\rm BT}}^{(0)}$ may be suboptimal for $V_{k+n+1}^*(\mathcal U^{(0)})$.
Using the fact that $\bfbeta(s_{n+1}){-}\bfbeta(s_n){>}0$, it follows
that
$V_{k}^{(\mathrm{BT})}(\mathcal U,{\mathcal S}_{{\rm BT}}^{(0)}){\geq} V_{k}^{(\mathrm{BT})}(\mathcal U,{\mathcal S}_{{\rm BT}}^{(1)})$, and 
${\mathcal S}_{{\rm BT}}^{(1)}$ is suboptimal.
The induction step, hence the theorem, are proved.
\end{proof}
\vspace{-6mm}
\bibliographystyle{IEEEtran}
\bibliography{IEEEabrv,bibliography}

\begin{IEEEbiography}
[{\includegraphics[width=1in,height=1.25in,clip,keepaspectratio]{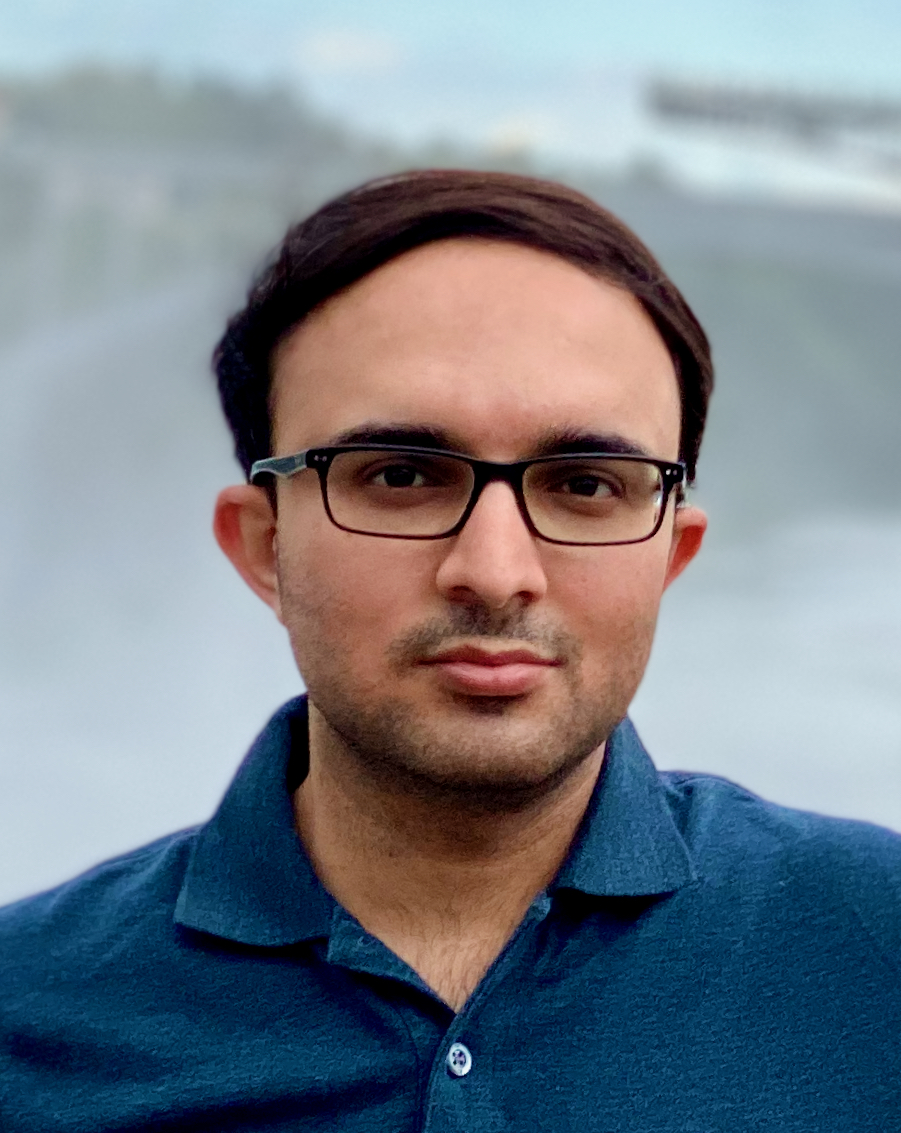}}]{Muddassar Hussain}
received   the   Bachelors in  electrical  engineering from  National University of Sciences and Technology  (NUST),  Islamabad,  Pakistan, in 2013. He received the Master's and PhD degrees in electrical and computer engineering from Purdue University, West Lafayette, IN, USA, in 2019 and 2021, respectively.  Currently, he is a senior systems engineer at Qualcomm Technologies. His research interest lie in the areas of optimization algorithms design,  stochastic optimal control, machine learning and reinforcement learning with applications to 5G wireless communication system design. He is reviewer of several IEEE journals and conferences, including IEEE Transactions on Wireless Communications, IEEE Transactions on Signal Processing, IEEE Transactions on Vehicular Technologies, IEEE Globecom, IEEE ICC. 
\end{IEEEbiography}
\begin{IEEEbiography}[{\includegraphics[width=1in,height=1.25in,clip,keepaspectratio]{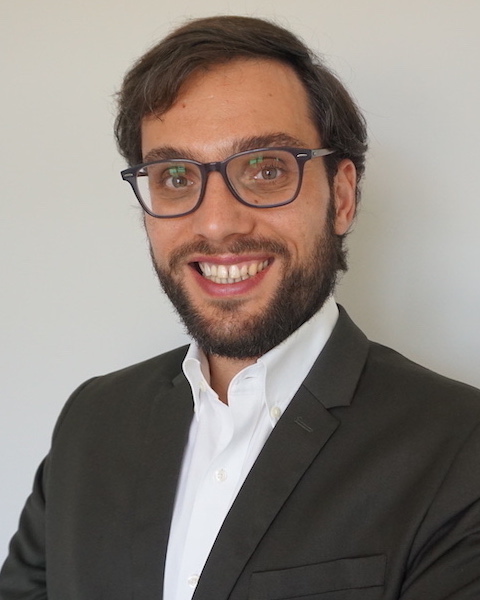}}]{Nicol\`{o} Michelusi}
(Senior Member, IEEE) received the B.Sc. (with honors), M.Sc. (with honors), and Ph.D. degrees from the University of Padova, Italy, in 2006, 2009, and 2013, respectively, and the M.Sc. degree in telecommunications engineering from the Technical University of Denmark, Denmark, in 2009, as part of the T.I.M.E. double degree program. From 2013 to 2015, he was a Postdoctoral Research Fellow with the Ming-Hsieh Department of Electrical Engineering, University of Southern California, Los Angeles, CA, USA, and from 2016 to 2020, he was an Assistant Professor with the School of Electrical and Computer Engineering, Purdue University, West Lafayette, IN, USA. He is currently an Assistant Professor with the School of Electrical, Computer and Energy Engineering, Arizona State University, Tempe, AZ, USA. His research interests include 5G wireless networks, millimeter-wave communications, stochastic optimization, distributed optimization, and federated learning over wireless. He is currently an Associate Editor for the IEEE TRANSACTIONS ON WIRELESS COMMUNICATIONS, and a Reviewer for several IEEE journals. He was the Co-Chair for the Distributed Machine Learning and Fog Network workshop at the IEEE INFOCOM 2021, the Wireless Communications Symposium at the IEEE Globecom 2020, the IoT, M2M, Sensor Networks, and Ad-Hoc Networking track at the IEEE VTC 2020, and the Cognitive Computing and Networking symposium at the ICNC 2018. He received the NSF CAREER award in 2021.
\end{IEEEbiography}

\end{document}